%% file: colt2025-sample.tex
\documentclass[final,12pt]{colt2025} 
\include{command}
\usepackage{algorithm}
\usepackage{algpseudocode}
\title[Alternating Regret for Online Convex Optimization]{Alternating Regret for Online Convex Optimization}

\usepackage{times}

\coltauthor{%
 \Name{Soumita Hait\nametag{\thanks{Authors are listed in alphabetical order.}}} \Email{hait@usc.edu}\\
 \addr University of Southern California
 \AND
 \Name{Ping Li\nametag{\footnotemark[1]}} \Email{pinglee@stu.sufe.edu.cn}\\
 \addr Shanghai University of Finance and Economics%
 \AND
 \Name{Haipeng Luo\nametag{\footnotemark[1]}} \Email{haipengl@usc.edu}\\
 \addr University of Southern California%
 \AND
 \Name{Mengxiao Zhang\nametag{\footnotemark[1]}} \Email{mengxiao-zhang@uiowa.edu}\\
 \addr University of Iowa
}

\begin{document}

\maketitle
\begingroup
  \renewcommand\thefootnote{$\dagger$}
  \footnotetext{Accepted for presentation at the Conference on Learning Theory (COLT)~2025.}%
  \addtocounter{footnote}{0}
\endgroup
\begin{abstract}%
Motivated by alternating learning dynamics in two-player games, a recent work by \citet{cevher2024alternation} shows that $o(\sqrt{T})$ alternating regret is possible for any $T$-round adversarial Online Linear Optimization (OLO) problem, and left as an open question whether the same is true for general Online Convex Optimization (OCO).
We answer this question in the affirmative by showing that 
the continuous Hedge algorithm achieves $\tilde{\mathcal{O}}(d^{\frac{2}{3}}T^{\frac{1}{3}})$ alternating regret for any adversarial $d$-dimensional OCO problems.
We show that this implies an alternating learning dynamic that finds a Nash equilibrium for any convex-concave zero-sum games or a coarse correlated equilibrium for any convex two-player general-sum games at a rate of $\tilde{\mathcal{O}}(d^{\frac{2}{3}}/T^{\frac{2}{3}})$.
To further improve the time complexity and/or the dimension dependence, we propose another simple algorithm, Follow-the-Regularized-Leader with a regularizer whose convex conjugate is 3rd-order smooth, for OCO with smooth and self-concordant loss functions (such as linear or quadratic losses).
We instantiate our algorithm with different regularizers and show that, for example, when the decision set is the $\ell_2$ ball, our algorithm achieves $\tilde{\mathcal{O}}(T^{\frac{2}{5}})$ alternating regret with no dimension dependence (and a better $\tilde{\mathcal{O}}(T^{\frac{1}{3}})$ bound for quadratic losses).
We complement our results by showing some algorithm-specific alternating regret lower bounds, including a somewhat surprising $\Omega(\sqrt{T})$ lower bound for a Regret Matching variant that is widely used in alternating learning dynamics.
\end{abstract}

\input{intro}
\input{preliminary}
\input{upper_bound}

\input{FTRL_SC}
\input{lower_bound}

\input{conclusion}

\bibliography{ref}
\newpage
\appendix

\input{appendix_pre}
\input{appendix_general_oco}
\input{appendix_scb}

\input{appendix_alt_oco}
\input{appendix_lower_bound}

\end{document}

%% file: command.tex
\newif\ifspacehack
\usepackage{natbib}
\usepackage{hyperref}
\hypersetup{
    colorlinks = blue,
    breaklinks,
    linkcolor = blue,
    citecolor = blue,
    urlcolor  = blue,
}
\usepackage{url} 
\usepackage{graphicx}
\usepackage{mathtools}
\usepackage{footnote}
\usepackage{float}
\usepackage{xspace}
\usepackage{multirow}
\usepackage{xcolor}
\usepackage{wrapfig}
\usepackage{framed}
\usepackage{bbm}
\usepackage[most]{tcolorbox}

\usepackage{footnote}
\usepackage{nicefrac}
\usepackage{makecell}
\usepackage{amssymb}
\usepackage{bm}
\makesavenoteenv{tabular}
\makesavenoteenv{table}

\newcommand{\interior}{\ensuremath{\text{Int}}}
\newcommand\innerp[2]{\langle #1, #2 \rangle}
\renewcommand{\tilde}{\widetilde}
\renewcommand{\hat}{\widehat}

\def \R {\mathbb{R}}

\newcommand{\calB}{{\mathcal{B}}}
\newcommand{\calX}{{\mathcal{X}}}

\newcommand{\calP}{{\mathcal{P}}}

\newcommand{\calY}{{\mathcal{Y}}}

\newcommand{\Reg}{{\mathrm{Reg}}}
\newcommand{\RegAlt}{\mathrm{Reg}_{\mathrm{Alt}}}
\newcommand{\RegCht}{{\mathrm{Reg}_{\mathrm{Cht}}}}
\newcommand{\Alg}{{\mathsf{Alg}}}

\newcommand{\E}{{\mathbb{E}}}

\newcommand{\inner}[1]{ \left\langle {#1} \right\rangle }

\newcommand{\order}{\mathcal{O}}

\renewcommand{\ln}{\log}

\newcommand{\Var}{\mathrm{Var}}

\newcommand{\KL}{\text{\rm KL}}

\DeclareMathOperator*{\argmin}{argmin}

\newcommand{\field}[1]{\mathbb{#1}}

\newcommand{\fR}{\field{R}}

\newcommand{\norm}[1]{\left\|{#1}\right\|}

\newcommand{\wh}{\widehat}

\newcommand{\intO}{\mathrm{int}(\Omega)}

\newcommand{\otil}{\ensuremath{\tilde{\mathcal{O}}}}

\newcommand{\rbr}[1]{\left(#1\right)}
\newcommand{\sbr}[1]{\left[#1\right]}
\newcommand{\cbr}[1]{\left\{#1\right\}}

\usepackage{lipsum,booktabs}
\usepackage{amsmath,mathrsfs,amssymb,amsfonts,bm,enumitem}
\usepackage{rotating}
\usepackage{pdflscape}
\usepackage{hyperref,url}
\hypersetup{
    colorlinks,
    breaklinks,
    linkcolor = blue,
    citecolor = blue,
    urlcolor  = blue,
}
\allowdisplaybreaks
\usepackage{appendix}
\usepackage{multirow,makecell}

\renewcommand{\tilde}{\widetilde}
\renewcommand{\hat}{\widehat}

\def \E {\mathbb{E}}

\def \R {\mathbb{R}}

\def \T {\top}

\usepackage{mathtools}
\let\norm\undefined % <-- "Undefine" \norm
\DeclarePairedDelimiter\norm{\lVert}{\rVert}
\DeclarePairedDelimiter\abs{\lvert}{\rvert}
\newtheorem{assumption}{Assumption}

\newtcolorbox{promptbox}[1][]{
  colback=blue!5!white, colframe=blue!75!black,
  fonttitle=\bfseries, title=Prompt,
  left=2mm, right=2mm, top=2mm, bottom=2mm,
  boxrule=0.5mm,  % Thickness of the frame
  coltitle=black, % Color of the title text
  colbacktitle=blue!15!white, % Background color of the title
  breakable,      % Allows the box to break across pages
  #1
}

\usepackage{graphicx,color} % more modern

\definecolor{wine_red}{RGB}{228,48,64}
\definecolor{DSgray}{cmyk}{0,1,0,0}

\usepackage{prettyref}
\newcommand{\pref}[1]{\prettyref{#1}}

\newcommand{\savehyperref}[2]{\texorpdfstring{\hyperref[#1]{#2}}{#2}}
\newrefformat{eq}{\savehyperref{#1}{Eq. \textup{(\ref*{#1})}}}
\newrefformat{eqn}{\savehyperref{#1}{Eq.~(\ref*{#1})}}
\newrefformat{lem}{\savehyperref{#1}{Lemma~\ref*{#1}}}
\newrefformat{event}{\savehyperref{#1}{Event~\ref*{#1}}}
\newrefformat{def}{\savehyperref{#1}{Definition~\ref*{#1}}}
\newrefformat{line}{\savehyperref{#1}{Line~\ref*{#1}}}
\newrefformat{thm}{\savehyperref{#1}{Theorem~\ref*{#1}}}
\newrefformat{tab}{\savehyperref{#1}{Table~\ref*{#1}}}
\newrefformat{corr}{\savehyperref{#1}{Corollary~\ref*{#1}}}
\newrefformat{cor}{\savehyperref{#1}{Corollary~\ref*{#1}}}
\newrefformat{sec}{\savehyperref{#1}{Section~\ref*{#1}}}
\newrefformat{app}{\savehyperref{#1}{Appendix~\ref*{#1}}}
\newrefformat{assum}{\savehyperref{#1}{Assumption~\ref*{#1}}}
\newrefformat{asm}{\savehyperref{#1}{Assumption~\ref*{#1}}}
\newrefformat{asp}{\savehyperref{#1}{Assumption~\ref*{#1}}}
\newrefformat{ex}{\savehyperref{#1}{Example~\ref*{#1}}}
\newrefformat{fig}{\savehyperref{#1}{Figure~\ref*{#1}}}
\newrefformat{alg}{\savehyperref{#1}{Algorithm~\ref*{#1}}}
\newrefformat{rem}{\savehyperref{#1}{Remark~\ref*{#1}}}
\newrefformat{conj}{\savehyperref{#1}{Conjecture~\ref*{#1}}}
\newrefformat{prop}{\savehyperref{#1}{Proposition~\ref*{#1}}}
\newrefformat{proto}{\savehyperref{#1}{Protocol~\ref*{#1}}}
\newrefformat{prob}{\savehyperref{#1}{Problem~\ref*{#1}}}
\newrefformat{claim}{\savehyperref{#1}{Claim~\ref*{#1}}}
\newrefformat{que}{\savehyperref{#1}{Question~\ref*{#1}}}
\newrefformat{op}{\savehyperref{#1}{Open Problem~\ref*{#1}}}
\newrefformat{fn}{\savehyperref{#1}{Footnote~\ref*{#1}}}

\def \epsilon {\varepsilon}

\def \interior {\text{int}}

\newcommand{\PRM}{PRM$^+$~}

%% file: intro.tex
\section{Introduction}
\label{intro}

Online Convex Optimization (OCO) is a heavily-studied framework for online learning~\citep{zinkevich2003online}.
In this framework, a learner sequentially interacts with an adversary for $T$ rounds.
In each round $t$, the learner selects an action $x_t$ from a fixed decision space $\calX$,
after which the adversary decides a convex loss function $f_t: \calX\mapsto \R$.
The learner then suffers loss $f_t(x_t)$ and observes $f_t$.
The standard performance measure for the learner is the worst-case \emph{regret}, defined as the difference between their total loss over $T$ rounds and that of a best fixed action in hindsight.

One important application of OCO algorithms is for learning in large-scale games and finding approximate equilibria efficiently.
The fundamental connection between online learning and game solving dates back to~\citep{foster1997calibrated,freund1999adaptive, hart2000simple} and has been driving some of the recent AI breakthroughs such as superhuman-level poker agents~\citep{bowling2015heads, moravvcik2017deepstack, brown2018superhuman, brown2019superhuman}.
The most basic result in this area is that in a two-player zero-sum convex-concave game, if both players repeatedly play the game using an OCO algorithm to make their decisions, then their time-averaged strategy is an approximate Nash Equilibrium (NE) with the duality gap being the average regret of the OCO algorithm. 
Since the worst-case regret for most problems is $\order(\sqrt{T})$, this implies an $\order(1/\sqrt{T})$ convergence rate to NE.

There has been a surge of studies on how to improve over this $\order(1/\sqrt{T})$ convergence rate in recent years.
One of the simplest approaches is to apply \textit{alternation}, which requires minimal change to the original approach: let the two players make their decisions in turn instead of simultaneously.
Such a simple adjustment turns out to be highly efficient in speeding up the convergence and is a default method in solving large-scale extensive-form games~\citep{tammelin2014solving}.

However, why alternation works so well has been a mystery until some recent progress.
Specifically, \citet{wibisono2022alternating} show that in normal-form games (where losses are linear), applying alternation with standard algorithms such as Hedge~\citep{freund1997decision} achieves a faster convergence rate of $\order(1/T^{\frac{2}{3}})$.
To analyze alternation, they propose a new regret notion, later termed \textit{alternating regret} by~\citet{cevher2024alternation}, and show that the convergence rate in an alternating learning dynamic is governed by the average alternating regret of the two players. 
In a word, alternating regret can be seen as the sum of standard regret plus \textit{cheating regret} where the learner's loss at time $t$ is evaluated at the next decision $x_{t+1}$, that is, $f_t(x_{t+1})$.
Since $x_{t+1}$ is computed with the knowledge of $f_t$,
cheating regret can be easily made negative, making alternating regret potentially much smaller than standard regret and explaining the faster $\order(1/T^{\frac{2}{3}})$ convergence.

Motivated by this work, \citet{cevher2024alternation} further investigate alternating regret for the special case of Online Linear Optimization (OLO) where all loss functions are linear.
They show that, even when facing a \textit{completely adversarial} environment, alternating regret can be made as small as $\order(T^{\frac{1}{3}})$ when the decision set $\calX$ is a simplex, or even $\order(\log T)$ when $\calX$ is an $\ell_2$-ball (while the standard regret for both cases is $\Omega(\sqrt{T})$).
The latter result immediately implies that applying their algorithm to an alternating learning dynamic for a zero-sum game defined over an $\ell_2$-ball leads to $\otil(1/T)$ convergence.

However, the results of \citet{cevher2024alternation} are restricted to linear losses, and they explicitly ask the question whether similar results can be derived for convex losses.
Indeed, while for standard regret, a convex loss $f_t$ can be reduced to a linear loss $x\mapsto \inner{\nabla f_t(x_t), x}$, this is no longer the case for alternating regret (see \pref{sec:adv-oco} for detailed explanation).

\paragraph{Contributions.} In this work, we answer this open question in the affirmative and provide several results that significantly advance our understanding on alternating regret and alternating learning dynamics for two-player games (both zero-sum and general sum).
Our concrete contributions are:
\begin{itemize}[leftmargin=*]
    \item In \pref{sec: alt-oco}, we start with an observation that the vanilla Hedge algorithm~\citep{freund1997decision} already achieves $\order(T^\frac{1}{3}\log^{\frac{2}{3}}d)$ alternating regret for an adversarial $d$-expert problem (that is, OLO over a simplex), a result that is implicitly shown in~\citet{wibisono2022alternating}.
    This observation not only already improves over the $\order(T^\frac{1}{3}\log^{\frac{4}{3}}(dT))$ bound of the arguably more complicated algorithm by~\citet{cevher2024alternation},
    but more importantly also inspires us to consider the \emph{continuous Hedge} algorithm for general OCO. 
    Indeed, via a nontrivial generalization of the analysis, we manage to show that continuous Hedge achieves $\otil(d^{\frac{2}{3}}T^\frac{1}{3})$ alternating regret for any convex decision set $\calX$ and any bounded convex loss functions.

    \item This general result implies that for any convex-concave zero-sum games, applying continuous Hedge in an alternating learning dynamic leads to $\otil(d^{\frac{2}{3}}/T^{\frac{2}{3}})$ convergence to NE.
    Moreover, we generalize this implication to two-player general-sum games with convex losses, showing the same convergence rate for coarse correlated equilibria.
    As far as we know, this is the first result for alternating learning dynamics in general-sum games.

    \item Although continuous Hedge can be implemented in polynomial time using existing log-concave samplers, the time complexity can be high.
    Moreover, its alternating regret bound depends on the dimension $d$ even when $\calX$ is an $\ell_2$ ball.
    To address these issues, in \pref{sec:SC}, we propose another simple algorithm that applies
    Follow-the-Regularized-Leader (FTRL) to the original convex losses (instead of their linearized version) with a Legendre regularizer whose convex conjugate is third-order smooth.
    Based on the decision set, we instantiate the algorithm with different regularizers and analyze its alternating regret for OCO with smooth and self-concordant loss functions (which include linear and quadratic losses).
    For example, when $\calX$ is an $\ell_2$ ball, our algorithm achieves $\otil(T^\frac{2}{5})$ alternating regret without any dimension dependence, which can be further improved to $\otil(T^{\frac{1}{3}})$ for quadratic losses.

    \item Finally, in \pref{sec: lower_bound}, we discuss some algorithm-specific alternating regret lower bounds, focusing on the special case with linear losses over a simplex.
    We start by showing that $\Omega(T^{\frac{1}{3}})$ alternating regret is unavoidable for Hedge, at least not with a fixed learning rate.
    We then turn our attention to a widely-used algorithm, called Predictive Regret Matching$^+$~\citep{farina2021faster},
    and show that despite its amazing empirical performance in alternating learning dynamics, 
    its alternating regret is in fact $\Omega(\sqrt{T})$ when facing an adversarial environment.
    This implies that to demystify its excellent performance, one has to take the game structure into account. 
    Moreover, using the same loss sequence, we show that another standard algorithm called Optimistic Online Gradient Descent (OOGD) \citep{rakhlin2013optimization} also suffers $\Omega(\sqrt{T})$ alternating regret.
\end{itemize}

\paragraph{Related Work.}
We refer the reader to~\citet{hazan2016introduction, orabona2019modern} for a general introduction to OCO.
The continuous Hedge algorithm dates back to Cover's universal portfolio algorithm~\citep{cover1991universal}, and is known to achieve $\otil(\sqrt{dT})$ (standard) regret for general convex losses (see e.g.,~\citet{narayanan2010random}) and $\order(d\log T)$ (standard) regret for exp-concave losses \citep{hazan2007logarithmic}.
OCO with self-concordant loss functions have been studied in~\citet{zhang2017improved} under the context of dynamic regret.

As mentioned, there is a surge of studies on learning dynamics that converge faster than $1/\sqrt{T}$, such as~\citet{daskalakis2011near, rakhlin2013optimization, syrgkanis2015fast, chen2020hedging, daskalakis2021near-optimal, farina2022kernelized, anagnostides2022near-optimal, anagnostides2022uncoupled, farina2022near}.
All these works achieve acceleration by improving the standard $\order(\sqrt{T})$ regret bound, which is only possible by exploiting the fact that the environment is not completely adversarial but controlled by other players who deploy a similar learning algorithm.

Alternation, on the other hand, is a much simpler approach to achieve acceleration, but the reason why it works so well is not well understood.
Earlier works can only show that it works at least as well as the standard simultaneous dynamics~\citep{tammelin2014solving, burch2019revisiting} or it works strictly better but without a good characterization on how much better~\citep{grand2024solving}.
\citet{wibisono2022alternating} are the first to show a substantial convergence improvement brought by alternation via the notion of alternating regret.
Their $\order(1/T^{\frac{2}{3}})$ convergence rate was later improved to $\order(1/T^{\frac{4}{5}})$ by~\citet{katona2024symplectic}.
While this latter result heavily relies on the game structure, as we point out in~\pref{sec: alt-oco}, the alternating regret bound of~\citet{wibisono2022alternating} in fact holds even in the adversarial setting.
Focusing on OLO, \citet{cevher2024alternation} further provide more results on $o(\sqrt{T})$ alternating regret in an adversarial environment.
These results make the analysis for alternating learning dynamics much simpler since one can simply ignore how the opponent behaves and how the players' decisions are entangled. 
Our work significantly advances our understanding on how far this approach can go by generalizing their results to convex losses as well as providing lower bounds for popular algorithms.

%% file: preliminary.tex
\section{Preliminaries}

\paragraph{General Notations.} For a convex function $f:\calX\mapsto\R$ and $x,y\in\calX\subseteq\R^d$, define the Bregman divergence between $x$ and $y$ with respect to $f$ as $D_{f}(x,y)=f(x)-f(y)-\inner{\nabla f(y),x-y}$. Define the convex conjugate of $f$, denoted by $f^*:\R^d\mapsto\R$, as $
f^*(y) = \sup_{x \in \calX} \left\{ \inner{x,y} - f(x) \right\}$.
For a vector $v\in\R^d$, denote $v_i$ as the $i$-th entry of the vector. Define $\R^d_+$ to be the $d$-dimensional Euclidean space in the positive orthant and $\Delta_d=\{x\in\R_+^d:~\sum_{i=1}^dx_i=1\}$ to be the $(d-1)$-dimensional simplex. We also denote $\Delta_{\mathcal{X}}$ to be the set of probability densities over a convex domain $\mathcal{X} \subset \mathbb{R}^d$, i.e., 
$\Delta_{\mathcal{X}} = \left\{ p: \mathcal{X} \to \mathbb{R}_+ \mid \int_{x\in\mathcal{X}} p(x) \, dx = 1 \right\}$. 
Let $\calB_2^d(1)=\{x\in\R^d,\|x\|_2\leq 1\}$ be the $\ell_2$ unit ball, and $e_i$ be the one-hot vector in an appropriate dimension with the $i$-th entry being $1$ and remaining entries being $0$.  For a bounded convex domain $\calX$, define its diameter as $\max_{x',x\in\calX}\|x-x'\|_2$. Define $\interior(\calX)$ as the interior of a convex domain $\calX$ and $\partial(\calX)$ as the boundary of $\calX$. 
The notation $\otil(\cdot)$ is used to hide all logarithmic terms.

\subsection{Adversarial Online Convex Optimization}\label{sec:adv-oco}
In adversarial Online Convex Optimization (OCO), at each round $t\in[T]$, the learner picks an action $x_t\in\calX\subseteq\R^d$ for some convex domain $\calX$, and then the adversary picks a convex loss function $f_t: \calX \mapsto \R$ and reveals it to the learner. The standard regret measures the difference between the total loss suffered by the learner and that of a best fixed decision in hindsight: $\Reg \triangleq \max_{u\in\calX}\Reg(u)$ where
$
    \Reg(u) \triangleq \sum_{t=1}^Tf_t(x_t) - \sum_{t=1}^Tf_t(u).
$
On the other hand, the main focus of our work is \emph{alternating regret}~\citep{wibisono2022alternating, cevher2024alternation}, defined as follows: $\RegAlt \triangleq \max_{u\in\calX}\RegAlt(u)$ where
\begin{align*}
    \RegAlt(u) \triangleq \sum_{t=1}^T\left(f_t(x_t)+f_t(x_{t+1})\right) - 2\sum_{t=1}^Tf_t(u)
\end{align*}
is the sum of standard regret $\Reg(u)$ and \emph{cheating regret}
$
\RegCht(u) \triangleq \sum_{t=1}^T f_t(x_{t+1}) - \sum_{t=1}^Tf_t(u).
$\footnote{For notational convenience, we include $x_{T+1}$, the output of the algorithm at the beginning of round $T+1$, in the definition.}
The cheating regret evaluates the learner's loss at time $t$ using the \emph{next} decision $x_{t+1}$ that is computed with the knowledge of $f_t$ (hence cheating),
and thus intuitively can be easily made negative.
The hope is that it is negative enough so that the alternating regret is order-different from the standard regret.
Indeed, for the special case of Online Linear Optimization (OLO) where all $f_t$'s are linear functions, \citet{cevher2024alternation} show a seperation between standard regret $\Reg$ and alternating regret $\RegAlt$:
while the minimax bound for the former is $\Theta(\sqrt{T})$,
the latter can be as small as $\otil(T^{\frac{1}{3}})$ when $\calX$ is a simplex and $\order(\log T)$ when $\calX$ is an $\ell_2$-ball.

However, as explicitly mentioned in~\citet{cevher2024alternation}, extension to general convex functions is highly unclear.
To see this, first recall the standard OCO to OLO reduction for standard regret:
$\Reg(u) \leq \sum_{t=1}^T \inner{\nabla f_t(x_t), x_t-u}$, which is due to the convexity of $f_t$ and reveals that it is sufficient to consider an OLO instance with linear loss $x\mapsto\inner{\nabla f_t(x_t), x}$ at time $t$.
However, when applying the same trick to alternating regret: that is, bounding $\RegAlt(u)$ by
$
 \sum_{t=1}^T \inner{\nabla f_t(x_t), x_t-u} + \inner{\nabla f_t(x_{t+1}), x_{t+1}-u},
$
one quickly realizes that this does not correspond to the alternating regret of any OLO instance.
In particular, the alternating regret for an OLO instance with linear loss $x\mapsto\inner{\nabla f_t(x_t), x}$ at time $t$ is instead: 
$
 \sum_{t=1}^T \inner{\nabla f_t(x_t), x_t-u} + \inner{\nabla f_t(x_{t}), x_{t+1}-u}.
$

\subsection{Alternating Learning Dynamics in Games}\label{sec:alt-game}
The motivation of studying alternating regret stems from learning in games with \textit{alternation}. Specifically, consider a two-player general-sum game 
where $u_1(x,y), u_2(x,y):\calX\times\calY\mapsto[-1,1]$ are the loss functions for $x$-player and $y$-player, respectively. $u_1(x,y)$ is convex in $x$ for any $y\in\calY$, and $u_2(x,y)$ is convex in $y$ for any $x\in\calX$. 
A pair of strategies $(\bar{x}, \bar{y})$ is an $\epsilon$-NE (Nash Equilibrium) if no player has more than $\epsilon$ incentive to unilaterally deviate from it: $u_1(\bar{x}, \bar{y}) \leq \min_{x\in\calX} u_1(x, \bar{y}) + \epsilon$
and $u_2(\bar{x}, \bar{y}) \leq \min_{y\in\calY} u_2(\bar{x}, y) + \epsilon$.
A joint distribution $\calP$ over $\calX \times \calY$ is called an $\epsilon$-CCE (Coarse Correlated Equilibrium) if in expectation no player has more than $\epsilon$ incentive to unilaterally deviate from a strategy drawn from $\calP$ to a fixed strategy: $\E_{(x,y)\sim \calP}[u_1(x,y)] \leq \min_{x'\in\calX}\E_{(x,y)\sim \calP}[u_1(x',y)] + \epsilon$ and $\E_{(x,y)\sim \calP}[u_2(x,y)] \leq \min_{y'\in\calY}\E_{(x,y)\sim \calP}[u_2(x,y')] + \epsilon$.

An efficient and popular way to find these equilibria is via repeated play using online learning algorithms.
For example, in a standard (simultaneous) learning dynamic, both players use an OCO algorithm to make their decision $x_t \in \calX$ and $y_t \in \calY$ for round $t$ simultaneously and observe their loss function $u_1(x, y_t)$ and $u_2(x_t, y)$ respectively. 
After interacting for $T$ rounds, it is well known that the uniform distribution over $(x_1, y_1), \dots, (x_T, y_T)$ is an $\epsilon$-CCE with $\epsilon$ being the average regret.
Moreover, in the special case of zero-sum games ($u_2 = -u_1$), the average strategy $(\frac{1}{T}\sum_{t=1}^T x_t, \frac{1}{T}\sum_{t=1}^T y_t)$
is an $\epsilon$-NE for the same $\epsilon$ (see e.g.,~\citealp{freund1999adaptive}).
Since one can ensure $\order(\sqrt{T})$ standard regret, a convergence rate of $\order(1/\sqrt{T})$ to these equilibria is immediate.

There are many studies on how to accelerate such  $\order(1/\sqrt{T})$ convergence, with alternation probably being the simplest and most practically popular one.
Specifically, the difference in an alternating learning dynamic compared to a simultaneous one is that at each round $t$, the $x$-player first makes their decision $x_t$, and then the $y$-player, \emph{seeing $x_t$}, makes their decision $y_t$; see \pref{alg:alternation}.

\begin{algorithm}[tb]
   \caption{Alternating Learning Dynamic using OCO}
   \label{alg:alternation}

   {\bfseries Input:} a two-player game with loss functions $u_1, u_2: \calX\times\calY\mapsto[-1,1]$.
   
   {\bfseries Input:} OCO algorithms $\Alg_x$ (for $\calX$) and $\Alg_y$ (for $\calY$).
   
   \For{$t=1$ {\bfseries to} $T$}{
   
   $\Alg_x$ decides $x_t \in \calX$ using past loss functions $u_1(x, y_1), \dots, u_1(x, y_{t-1})$
   
   $\Alg_y$ decides $y_t \in \calY$ using past loss functions $u_2(x_1, y), \dots, u_2(x_{t-1}, y)$ and $u_2(x_t, y)$
   
   }
\end{algorithm}

In the special case of zero-sum games ($u_2=-u_1$),
\citet{wibisono2022alternating, cevher2024alternation} show that the average strategy of such an alternating learning dynamic is an $\epsilon$-NE with $\epsilon$ now being the average \emph{alternating} regret.
While they only prove this for linear losses, it is straightforward to generalize it to convex losses, which is included below for completeness.

\begin{theorem}\label{thm:zero-sum}
In the alternating learning dynamic described by \pref{alg:alternation} for a zero-sum game ($u_2=-u_1$), 
suppose that $\Alg_x$ and $\Alg_y$ 
guarantee a worst-case alternating regret bound $\RegAlt^x$ and $\RegAlt^y$ respectively.
Then, the average strategy $(\frac{1}{T}\sum_{t=1}^T x_t, \frac{1}{T}\sum_{t=1}^T y_t)$ is an $\epsilon$-NE with $\epsilon = \order\left(\frac{\RegAlt^x+\RegAlt^y}{T}\right)$.
\end{theorem}

Therefore, if alternating regret can be made $o(\sqrt{T})$, a faster than $1/\sqrt{T}$ convergence rate is achieved.
In fact, such an implication generalizes to CCE for two-player general-sum games as well, which is unknown before to the best of our knowledge.

\begin{theorem}\label{thm:general-sum}
In the alternating learning dynamic described by \pref{alg:alternation} (for general-sum games), suppose that $\Alg_x$ and $\Alg_y$ 
guarantee a worst-case alternating regret bound $\RegAlt^x$ and $\RegAlt^y$ respectively.
Then, the uniform distribution over $\{(x_t,y_t), (x_{t+1},y_t)\}_{t\in[T]}$ is an $\epsilon$-CCE with $\epsilon = \order\left(\frac{\max\{\RegAlt^x,\RegAlt^y\}}{T}\right)$.
\end{theorem}

Both theorems are a direct consequence of the definition of alternating regret; see \pref{app:alt-game} for the proofs.
We remark that it is unclear how to generalize \pref{thm:general-sum} to a general-sum game with more than two players.
We conjecture that such a generalization requires either a new alternation scheme or a new concept of regret.

Given these connections between alternating regret and the convergence of alternating learning dynamics, the rest of the paper focuses on understanding what alternating regret bounds are achievable in an adversarial environment.

%% file: upper_bound.tex
%!TEX root=main.tex
\section{$o(\sqrt{T})$ Alternating Regret for OCO}\label{sec: alt-oco}
In this section, we propose an OCO algorithm with $\otil(T^{\frac{1}{3}})$ alternating regret.
Our idea is based on the following observation for the special case of the expert problem where $\calX=\Delta_d$ and $f_t(x) = \inner{\ell_t,x}$, $\|\ell_t\|_{\infty}\leq 1$ for all $t\in[T]$. 
While~\citet{wibisono2022alternating} show that the vanilla Hedge algorithm~\citep{freund1997decision} achieves $\otil(T^{\frac{1}{3}})$ alternating regret when played against itself in a zero-sum game, 
the same result in fact holds for any adversarial expert problem, which was not made explicit in their work.
We summarize this observation in the following theorem.

\begin{theorem}[Implicit in~\citealp{wibisono2022alternating}]\label{thm:hedge_simplex}
For an adversarial OLO problem with $\calX=\Delta_d$ and $f_t(x) = \inner{\ell_t,x}$, $\|\ell_t\|_{\infty}\leq 1$ for all $t\in[T]$,
the Hedge algorithm that plays $p_t \in\Delta_d$ in round $t$ such that $p_{t,i} \propto \exp{(-\eta \sum_{\tau=1}^{t-1} \ell_{\tau,i})}$ with
$\eta = T^{-\frac{1}{3}}\log^{\frac{1}{3}}d$ ensures $\RegAlt=\order(T^{\frac{1}{3}}\log^{\frac{2}{3}}d)$.
\end{theorem}
\begin{proof}[sketch]
Via standard analysis of Hedge, we can show that the alternating regret against any $u$ is equal to the following 
\begin{align}\label{eqn:alt-regret-simplex-main}
    \frac{2(\KL(u,p_1)-\KL(u,p_{T+1}))}{\eta} + \frac{1}{\eta}\sum_{t=1}^T\left(\KL(p_t,p_{t+1})-\KL(p_{t+1},p_{t})\right).
\end{align}
Since $p_1$ is the uniform distribution over $d$ experts by definition, the first term $\frac{2(\KL(u,p_1)-\KL(u,p_{T+1}))}{\eta}$ can be bounded by $\frac{2\log d}{\eta}$. It remains to control the second term, which we call a $\KL$ commutator following the term Bregman commutator from~\citet{wibisono2022alternating}. Direct calculation shows that 
\begin{align*}
    \KL(p_t,p_{t+1})-\KL(p_{t+1},p_t) = D_{G^*}(-\eta L_t,-\eta L_{t-1}) - D_{G^*}(-\eta L_{t-1},-\eta L_{t}),
\end{align*}
where $G^*(w)=\log(\sum_{t=1}^T\exp(w_i))$ (the convex conjugate of negative entropy) and $L_t=\sum_{\tau\leq t}\ell_t$.
Then, applying the key Lemma A.2 of~\citet{wibisono2022alternating} (included as \pref{lem:bregman_commutator} here)
which bounds the Bregman commutator of a 3rd-order smooth function (see \pref{def:3rd-order smoothness}), we arrive at $D_{G^*}(-\eta L_t,-\eta L_{t-1}) - D_{G^*}(-\eta L_{t-1},-\eta L_{t})\leq \frac{4}{3}\|\eta\ell_t\|_{\infty}^3\leq \frac{4}{3}\eta^3$, since $G^*(x)$ is $8$-smooth of order 3 as proven in~\citet[Example A.3]{wibisono2022alternating}. Plugging this bound to \pref{eqn:alt-regret-simplex-main} and picking the optimal $\eta$ finishes the proof.
The full proof is deferred to \pref{app:CEW}.
\end{proof}

This bound not only improves upon the $\order(T^{\frac{1}{3}}\log^{\frac{4}{3}}d)$ guarantee of a much more complicated algorithm by~\citet{cevher2024alternation}, 
but perhaps more importantly, also inspires us to consider 
whether a continuous version of Hedge, called continuous Hedge (\pref{alg:hedge-cont}), can handle the general OCO setting well.
Indeed, for standard regret, continuous Hedge achieves an $\otil(\sqrt{dT})$ bound for any OCO instance~\citep{narayanan2010random}, generalizing Hedge's $\order(\sqrt{T\log d})$ standard regret from OLO to OCO.
Fortunately, in a similar way and via non-trivial analysis, we manage to show that the same generalization carries over for alternating regret, as shown below.

\begin{theorem}\label{thm:hedge-cont}
    For any OCO instance with  $\max_{x\in\calX}|f_t(x)|\le 1$ for all $t\in[T]$, \pref{alg:hedge-cont} with $\eta = \min\{1,T^{-\frac{1}{3}}(d\log T)^{\frac{1}{3}}\}$ achieves $\RegAlt=\order(d^{\frac{2}{3}}T^{\frac{1}{3}}\log^{\frac{2}{3}}T)$.
\end{theorem}

\begin{algorithm}[t]
   \caption{Continuous Hedge}
   \label{alg:hedge-cont}
   {\bfseries Input:} Parameter $\eta>0$
   
   \For{$t=1$ {\bfseries to} $T$}{
   
   The learner computes $p_t\in\Delta_{\calX}$ such that $p_t(x) \propto \exp{(-\eta \sum_{\tau=1}^{t-1} f_\tau(x))}$

   The learner chooses $x_t=\int_{x\in\calX}p_t(x)xdx$ and observes $f_t$.
   }
\end{algorithm}
The general idea of the proof follows that of \pref{thm:hedge_simplex}, which is to bound the alternating regret by the $\KL$ commutator $\sum_{t=1}^T(\KL(p_t,p_{t+1})-\KL(p_{t+1},p_{t}))$. The main technical part of the proof is then to further bound the per-round KL-divergence commutator $\KL(p_t,p_{t+1})-\KL(p_{t+1},p_t)$ by $\order(\eta^3)$ for the continuous distribution $p_t, p_{t+1}$ over the whole convex domain, which is done via a proof that deviates from that of~\citet[Lemma~A.2]{wibisono2022alternating}.
See \pref{app:CEW} for details.

We have thus shown that $\otil(d^{\frac{2}{3}}T^{\frac{1}{3}})$ alternating regret is achievable for general OCO, resolving the open question asked by~\citet{cevher2024alternation}
and bypassing the obstacle that OCO cannot be reduced to OLO in this case.
Note that we do not require either Lipschitzness or smoothness of the loss functions, but only bounded function value. 
Moreover, combining \pref{thm:hedge-cont} with \pref{thm:zero-sum} and \pref{thm:general-sum}, we immediately obtain the following convergence result on an alternating learning dynamic for general two-player games with convex loss functions.
\begin{corollary}\label{cor:game with entropic barrier}
    Consider an alternating learning dynamic described by \pref{alg:alternation} where the loss function $u_1(\cdot,y),u_2(x,\cdot)\in[-1,1]$ are convex for all $x\in\calX$ and $y\in\calY$, and $\Alg_x$ and $\Alg_y$ are \pref{alg:hedge-cont} with $\eta=(d\log T)^{\frac{1}{3}}T^{-\frac{1}{3}}$ (and $f_t(\cdot)$ being $u_1(\cdot, y_t)$ for $\Alg_x$ and $u_2(x_{t+1}, \cdot)$ for $\Alg_y$). Then the time-averaged strategy $(\frac{1}{T}\sum_{t=1}^T x_t, \frac{1}{T}\sum_{t=1}^T y_t)$ is an $\otil(d^{\frac{2}{3}}T^{-\frac{2}{3}})$-NE when the game is zero-sum; otherwise, the uniform distribution over the strategies $\{(x_t,y_t), (x_{t+1},y_t)\}_{t\in[T]}$ is an $\otil(d^{\frac{2}{3}}T^{-\frac{2}{3}})$-CCE.
\end{corollary}

While there exist (simultaneous) learning dynamics that achieve faster convergence for games with convex losses~\citep{syrgkanis2015fast, farina2022near},
our results are the first to show that better than $1/\sqrt{T}$ convergence is possible using simple alternating learning dynamics.

%% file: FTRL_SC.tex
\section{Another Algorithm for Smooth and Self-Concordant Losses}\label{sec:SC}

One issue with continuous Hedge is that its implementation requires a log concave sampler (e.g.,~\citet{lovasz2003geometry, lovasz2006simulated, narayanan2010random}), which can be expensive when the dimension $d$ is high (even though the time complexity is polynomial in $d$).
Another issue is that its alternating regret bound has dimension dependence even when $\calX$ is an $\ell_2$ ball where it is well-known that $d$-independent standard regret bound is possible.
To mitigate these issues, in this section, we propose yet another general algorithm and instantiate it for different decision spaces.
Our result here, however, only applies to a restricted class of convex losses.

Specifically, we start by providing some basic definitions on functions (throughout this section, all norms considered are $\ell_2$-norm).
\begin{definition}\label{def:lipschitz}
    We say that  a differentiable function $f:\calX\mapsto\R$ is $L$-Lipschitz 
    if for any $x\in\calX$, $\|\nabla f(x)\|_2 \leq L$,
    $\beta$-smooth  
    if for any $x,y\in\calX$, $D_f(x,y)\leq \frac{\beta}{2}\|x-y\|_2^2$,
     $\sigma$-strongly convex if for any $x,y\in\calX$, $D_f(x,y)\geq \frac{\sigma}{2}\|x-y\|_2^2$.
\end{definition}

\begin{definition}[3rd-order smoothness]\label{def:3rd-order smoothness}
A function $\psi:\calX\mapsto\R$ is said to be \emph{$M$-smooth of order 3} if it is three times differentiable, and its third-order derivative at any $w\in\calX$ satisfies
$
|\nabla^3 \psi(w)[h, h, h]|
\leq M,~ \forall h \in \mathbb{R}^d, \|h\|_2=1.
$
\end{definition}

\begin{definition}\label{def:Legendre}
    A convex function $f:\calX\mapsto \R$ for a convex body $\calX$ is Legendre if it is differentiable and strictly convex in $\interior(\calX)$ and also 
    $\norm{\nabla f(x)}_2\rightarrow \infty$ as $x\rightarrow\partial\calX$ (the boundary of $\calX$).    
\end{definition}

\begin{definition}[Self-concordance]\label{def:self-concordance}\citep{nesterov1994interior}
    A function $f:\calX\mapsto \fR$ is said to be $C$-self-concordant for $C\ge0$, if it is three times differentiable, and for all $x\in\calX$ and all $h\in\fR^d$, the following inequality holds: $
    \abs{\nabla^3f(x)[h,h,h]}\le2C\rbr{\nabla^2f(x)[h,h]}^{3/2}
    $.
\end{definition}

Our algorithm is the classical and efficient Follow-the-Regularized-Leader algorithm that, with a learning rate $\eta>0$ and a regularizer $\psi: \calX \mapsto \fR$,
plays 
$
x_{t}=\argmin_{x\in \mathcal{X}}\sum_{\tau=1}^{t-1} f_{\tau}(x)+\frac{1}{\eta}\psi(x)
$
in round $t$;
see \pref{alg:OCO}.
However, there are two important elements:
first, we apply FTRL directly to the original loss functions $f_\tau$ instead of its linearized version $x\mapsto \inner{\nabla f_\tau(x_\tau), x}$, in light of the issue discussed in \pref{sec:adv-oco};
second, it is critical that our regularizer $\psi$ satisfy the following conditions.

\begin{algorithm}[tb]
   \caption{FTRL for minimizing alternating regret}
   \label{alg:OCO}

   {\bfseries Input:} learning rate $\eta>0$.
   
   {\bfseries Input:} a regularizer $\psi$ satisfying \pref{asp: psi}.
   
   \For{$t=1$ {\bfseries to} $T$} {
   
   Play
   $
   x_{t}=\argmin_{x\in \mathcal{X}}\sum_{\tau=1}^{t-1} f_{\tau}(x)+\frac{1}{\eta}\psi(x)
   $
   and observe $f_t$.
   }

\end{algorithm}

\begin{assumption}\label{asp: psi}
    The regularizer $\psi$ is Legendre and $\sigma$-strongly convex  within domain $\calX$.
    Moreover, its convex conjugate $\psi^*(w) = \sup_{x\in\calX}\inner{w,x}-\psi(x)$ is $M$-smooth of order $3$. 
\end{assumption}

We will analyze \pref{alg:OCO} under the following assumption on the loss functions.
\begin{assumption}\label{asp: function}
    For all $t\in[T]$, the loss function $f_t$ is $L$-Lipschitz , $\beta$-smooth, and $C$-self-concordant.
\end{assumption}

For example, linear and convex quadratic loss functions satisfy \pref{asp: function} with $C=0$.
For quadratic functions in the form of $x^\T Ax+bx+c$, 
we note that applying Online Gradient Descent 
achieves $\order\rbr{\frac{d}{\alpha}\log T}$ standard regret~\citep{hazan2007logarithmic}, where $\alpha\propto\frac{\lambda_{\min}(A+A^\T)}{\lambda_{\max}(A+A^\T)}$ can be small if the matrix $A+A^\T$ is ill-conditioned,
while our alternating regret bound enjoys better dependence on $d$ and is independent of $\alpha$.
There are also examples where the quadratic loss is not even exp-concave 
(for example, when $A+A^\top$ has eigenvalue $0$ with a corresponding eigenvector $u$ such that  $b^\top u \neq 0$),
so Online Newton Step~\citep{hazan2007logarithmic} cannot apply, 
but our results still hold.
As mentioned, online learning with self-concordant losses has been considered in~\citet{zhang2017improved} under the context of dynamic regret.

While our algorithm appears to be standard and has been studied in the literature for standard regret (see e.g.,~\citet{orabona2019modern}),
analyzing its alternating regret requires new ideas based on the properties of the regularizer.
Specifically, we prove the following general alternating regret bound.
\begin{theorem}\label{thm: OCO}
    Under \pref{asp: function}, \pref{alg:OCO} guarantees the following bound on $\RegAlt(u)$ for any $u\in\calX$,
    \begin{equation}\label{eqn:OCO}
        \begin{aligned}
        &\mathcal{O}\rbr{\frac{\psi(u)-\psi(x_1)}{\eta}+ \frac{CL^3}{\sigma^{3/2}}\eta^{3/2}T+\rbr{ML^3+\frac{\beta L^2}{\sigma^2}}\eta^2T}.
        \end{aligned}
    \end{equation}
\end{theorem}

\begin{proof}
    Define $F_{t}(x)\triangleq\sum_{\tau=1}^{t-1}f_\tau(x)+\frac{1}{\eta}\psi(x)$, 
    so that $x_{t}=\argmin_{x\in\calX}F_{t}(x)$.
    Note that
\begin{align*}
    -\sum_{t=1}^Tf_t(u) 
    &= \frac{1}{\eta}\psi(u) - \frac{1}{\eta}\psi(x_1)+ F_{T+1}(x_{T+1}) - F_{T+1}(u) + \sum_{t=1}^T\left(F_t(x_t)-F_{t+1}(x_{t+1})\right)  \\
    &\leq \frac{\psi(u)-\psi(x_1)}{\eta} + \sum_{t=1}^T\left(F_t(x_t)-F_{t+1}(x_{t+1})\right),
\end{align*}
where 
the last inequality comes from $F_{T+1}(x_{T+1})\leq F_{T+1}(u)$ by the optimality of $x_{T+1}$.
Substituting it in the alternating regret definition, we get
\begin{align}
    \RegAlt(u) \leq  \frac{2\rbr{\psi(u)-\psi(x_1)}}{\eta} + \sum_{t=1}^T\rbr{2F_t(x_t)-2F_{t+1}(x_{t+1})+f_t(x_t)+f_t(x_{t+1})}. \label{eqn:RegAlt_bound}
\end{align}

Since $\psi$ is Legendre, we must have $\nabla F_t(x_t) = \mathbf{0}$ by the optimality of $x_t$, where $\mathbf{0}$ is the zero vector.
Further using the definition of Bregman divergence, we obtain
\begin{align*}
    &D_{F_{t+1}}(x_t,x_{t+1})
    = F_{t+1}(x_t) - F_{t+1}(x_{t+1})
    = F_t(x_t) + f_t(x_t) - F_{t+1}(x_{t+1})\\
    \text{and }&D_{F_{t}}(x_{t+1},x_{t})
    = F_t(x_{t+1}) - F_t(x_t) 
    = F_{t+1}(x_{t+1}) - f_t(x_{t+1}) - F_{t}(x_{t}).
\end{align*}
Therefore, the second term in the right-hand side of \pref{eqn:RegAlt_bound} can be written as
\begin{align*}
    &\sum_{t=1}^T\left(2F_t(x_t)-2F_{t+1}(x_{t+1})+f_t(x_t)+f_t(x_{t+1})\right)
    = \sum_{t=1}^T\left(D_{F_{t+1}}(x_t,x_{t+1}) - D_{F_t}(x_{t+1},x_t)\right) \\
    &=\sum_{t=1}^T\left(D_{f_t}(x_t,x_{t+1})+D_{F_{t}}(x_t,x_{t+1}) - D_{F_t}(x_{t+1},x_t)\right).
\end{align*}
Since $f_t$ is $\beta$-smooth, the first term above, $D_{f_t}(x_t,x_{t+1})$, is at most $\frac{\beta}{2}\norm{x_t-x_{t+1}}^2$,
which can be further bounded by $\frac{\beta\eta^2L^2}{2\sigma^2}$ due to a standard stability argument of FTRL (see \pref{lem:path_length}).
For the second and third terms, we proceed as
\begin{align*}
    &D_{F_t}(x_t,x_{t+1}) - D_{F_t}(x_{t+1},x_t) 
    \\
    &=  D_{{F}_t^*}(\nabla {F}_t(x_{t+1}),\nabla {F}_t(x_{t}))- D_{{F}_t^*}(\nabla {F}_t(x_{t}),\nabla {F}_t(x_{t+1})) \\
    &=  D_{{F}_t^*}(\nabla {F}_t(x_{t+1}),\mathbf{0}) - D_{{F}_t^*}(\mathbf{0},\nabla {F}_t(x_{t+1})) \tag{since $\nabla {F}_t(x_{t}) = \mathbf{0}$} \\
    &\leq \frac{1}{6}\rbr{\frac{2C\eta^{3/2}}{\sigma^{3/2}}+M\eta^2}\norm{\nabla {F}_t(x_{t+1})}_2^3 \\
    &\le \order\rbr{\rbr{\frac{C\eta^{3/2}}{\sigma^{3/2}}+M\eta^2} L^3}.
\end{align*}
Here, the first equality uses \pref{lem:conj_dual}, a basic property of Bregman divergence of convex conjugate;
the first inequality uses the fact that ${F}_t^*$ is $\rbr{\frac{2C\eta^{3/2}}{\sigma^{3/2}}+M\eta^2}$-smooth of order $3$ (see the technical \pref{lem:3rd-order-smooth-F}), the place where we require the self-concordance of the loss functions and the 3rd-order smoothness of $\psi^*$,
along with the aforementioned \pref{lem:bregman_commutator} taken from~\citet{wibisono2022alternating};
the last inequality comes from $\nabla F_t(x_{t+1})=\nabla {F}_{t+1}(x_{t+1}) - \nabla f_{t}(x_{t+1})=- \nabla f_{t}(x_{t+1})$ and the Lipschitzness of $f_{t}$.
Plugging all bounds into \pref{eqn:RegAlt_bound} finishes the proof.
\end{proof}

\begin{lemma}[Lemma A.2 of \citet{wibisono2022alternating}]\label{lem:bregman_commutator}
    Assuming $\psi^*$ is $M$-smooth of order $3$ with respect to norm $\|\cdot\|_2$, we have for all $w,w'\in\fR^d$,
    $\abs{D_{\psi^*}(w,w')-D_{\psi^*}(w',w)} \le \frac{M}{6}\norm{w-w'}_2^3$.
\end{lemma}

Focusing only on the dependence on $T$ and $C$, we see that \pref{thm: OCO} gives an alternating regret bound of $\otil(\frac{1}{\eta}+(C\eta^{3/2}+\eta^2) T)$, which is $\otil((CT)^{\frac{2}{5}}+T^{\frac{1}{3}})$ with the optimal tuning of $\eta$. When $C=0$, e.g., for linear and convex quadratic loss functions, we thus also have a bound of $\otil(T^{\frac{1}{3}})$;
otherwise, our bound is of order $\otil(T^{\frac{2}{5}})$.
To handle the dependence on all other parameters,
we provide concrete instantiation of $\psi$ in the rest of this section.

\subsection{Entropic Barrier Regularizer for General $\calX$}\label{sec:entopic_barrier}
In this section, we consider a general compact and convex body $\calX$ with a bounded diameter $D$.
Our observation is that the \emph{entropic barrier} of $\calX$, proposed by ~\citet{bubeck2014entropic}, is a valid regularizer for our algorithm (see \pref{app: alt-oco} for the proof).

\begin{lemma}\label{lem:entropic}
The convex conjugate of the entropic barrier for $\calX$,
$\psi^*(\theta) = \log\rbr{\int_{x\in\calX}\exp\rbr{\innerp{\theta}{x}}dx}$, satisfies \pref{asp: psi} with parameters $\sigma = \frac{1}{D^2}$ and $M=D^3$.
\end{lemma}

\begin{corollary}\label{cor:entr_regularizer}
    Let $\calX$ be a compact convex body with diameter $D$. Under \pref{asp: function}, \pref{alg:OCO} with $\eta=\min\cbr{(d\ln T)^{\frac{2}{5}}(CL^3D^3)^{-\frac{2}{5}}T^{-\frac{2}{5}},(d\ln T)^{\frac{1}{3}}(L^3D^3+\beta L^2D^2)^{-\frac{1}{3}}T^{-\frac{1}{3}}}$ and $\psi$ being the entropic barrier for $\calX$ guarantees
    \begin{equation*}
    \RegAlt\le\order\rbr{\max\cbr{(d\ln T)^{\frac{3}{5}}(CL^3D^3)^{\frac{2}{5}}T^{\frac{2}{5}},(d\ln T)^{\frac{2}{3}}(L^3D^3+\beta L^2D^2)^{\frac{1}{3}}T^{\frac{1}{3}}}}.
    \end{equation*}
\end{corollary}

\begin{proof}
    For any fixed $u\in\calX$, we decompose $\RegAlt(u)$ into 
    \begin{equation}\label{eqn:reg_decompose}      \RegAlt(u')+\sbr{2\sum_{t=1}^T\rbr{f_t(u')-f_t(u)}},
    \end{equation}
    where $u'=(1-\epsilon)u+\epsilon x_1\in\calX$ for $\epsilon=\frac{1}{T}$ and $x_1 = \argmin_{x\in\calX}\frac{\psi(x)}{\eta}$. 
    Since $f_t$ is convex, $f_t(u')-f_t(u) \le \frac{1}{T}\rbr{f_t(x_1)-f_t(u)}\le \frac{LD}{T}$ for all $t\in[T]$. Hence, the second term above can be bounded by $2LD$.
    For the first term $\RegAlt(u')$, we apply \pref{thm: OCO} together with \pref{lem:entropic} to bound it as
    \begin{equation}\label{eqn:reg_u'_bound}
         \order\rbr{\frac{\psi(u')-\psi(x_1)}{\eta}+CL^3D^3\eta^{3/2}T+\rbr{L^3D^3+\beta L^2D^2}\eta^2 T}.
    \end{equation}
Finally, since $\psi$ is an $\order(d)$ self-concordant barrier of $\calX$~\citep{bubeck2014entropic,chewi2023entropic}, we have $\psi(u')-\psi(x_1)= \order(d\ln T)$ (see \pref{lem:self_concordant}).
Plugging in the value of $\eta$ then finishes the proof.
\end{proof}
Again, for loss functions with self-concordance parameter $C=0$ (e.g., linear or convex quadratic losses), we achieve  $\otil(d^\frac{2}{3}T^\frac{1}{3})$ alternating regret, matching the performance of continuous Hedge.
In fact, for linear losses, the two algorithms are simply equivalent~\citep{bubeck2014entropic}, but this equivalence does not extend to general  convex losses.

\subsection{The $\ell_2$-Ball Case}\label{sec:ball}

Next, we consider the case where the feasible domain $\calX=\calB_2^d(1)$ 
and pick $\psi(x) = -\ln(1-\norm{x}_2^2)$. 
The following lemma (proof in \pref{app: alt-oco}) shows that $\psi$ satisfies \pref{asp: psi}.

\begin{lemma}\label{lem:smoothness of ball case}
    The convex conjugate of $\psi$ is   defined as
    $\psi^*(w) = \sqrt{1+\norm{w}_2^2} - \ln(1+\sqrt{1+\norm{w}_2^2})$,
    and it satisfies \pref{asp: psi} with parameters $\sigma = 2$ and $M=4$.
\end{lemma}
Thus, we get the following dimension independent regret bound from \pref{thm: OCO}.
\begin{corollary}\label{cor:ball}
    Let $\calX=\calB_2^d(1)$. Under \pref{asp: function}, \pref{alg:OCO} with regularizer $\psi(x) = -\ln(1-\norm{x}_2^2)$ and $\eta=\min\cbr{(\ln T)^{\frac{2}{5}}(CL^3)^{-\frac{2}{5}}T^{-\frac{2}{5}},(\ln T)^{\frac{1}{3}}(L^3+\beta L^2)^{-\frac{1}{3}}T^{-\frac{1}{3}}}$ guarantees
    $
    \RegAlt\le \order\rbr{\max\cbr{(\ln T)^{\frac{3}{5}}(CL^3)^{\frac{2}{5}}T^{\frac{2}{5}},(\ln T)^{\frac{2}{3}}(L^3+\beta L^2)^{\frac{1}{3}}T^{\frac{1}{3}}}}.
    $
\end{corollary}

\begin{proof}
    We decompose $\RegAlt(u)$ for any $u\in\calX$, as shown in \pref{eqn:reg_decompose}, and choose $u'$ similarly. The second term can be bounded by $4L$ and the first term by \pref{eqn:reg_u'_bound} again. Since $\psi^*$ is a $1$-self-concordant barrier (shown in section 6.2.1 of \citep{nesterov1994interior}), we have $\psi(u')-\psi(x_1)=\order(\ln T)$ by \pref{lem:self_concordant}. 
    Finally, plugging in $\sigma=2$ and $M=4$ and using the value of $\eta$ leads us to the final regret bound.
\end{proof}

For losses with $C=0$, we thus obtain a bound of order $\otil(T^\frac{1}{3})$, improving over the $\otil(d^{\frac{2}{3}}T^\frac{1}{3})$ bound of continuous Hedge.

\subsection{The Simplex Case}\label{sec:simplex}
Finally, we consider the case where the feasible domain is the $(d-1)$-dimensional simplex $\Delta_d$. Specifically, to represent $\Delta_d$ as a convex body in $(d-1)$-dimensional space, we set $\calX=\{x\in \R_+^{d-1}, \sum_{i=1}^{d-1}x_i\leq 1\}$. We instantiate \pref{alg:OCO} with the entropy regularizer $\psi(x) = \sum_{i=1}^{d-1} x_i \ln x_i + (1-\sum_{i=1}^{d-1}x_i)\ln (1-\sum_{i=1}^{d-1}x_i)$. The following lemma (proven in \pref{app: alt-oco}) shows that $\psi$ satisfies \pref{asp: psi} with $\sigma = 1$ and $M=8$. 
\begin{lemma}\label{lem:negative entropy}
    $\psi(x)$ is Legendre and $1$-strongly convex with respect to $\ell_2$-norm. The convex conjugate of $\psi(x)$ is $\psi^*(w)=\ln\rbr{1+\sum_{i=1}^{d-1} e^{w_i}}$, which is $8$-smooth of order $3$ with respect to $\ell_2$-norm.
\end{lemma}
\begin{corollary}\label{cor: simplex}
Let $\calX=\{x\in \R_+^{d-1}, \sum_{i=1}^{d-1}x_i\leq 1\}$. Under \pref{asp: function}, \pref{alg:OCO} with $\eta =\min\cbr{(\ln d)^{\frac{2}{5}}(CL^3)^{-\frac{2}{5}}T^{-\frac{2}{5}},(\ln d)^{\frac{1}{3}}(L^3+\beta L^2)^{-\frac{1}{3}}T^{-\frac{1}{3}}}$ and regularizer $\psi(x)= \sum_{i=1}^{d-1} x_i \ln x_i + (1-\sum_{i=1}^{d-1}x_i)\ln (1-\sum_{i=1}^{d-1}x_i)$ guarantees 
\[
\RegAlt \le \order\rbr{\max\cbr{(\ln d)^{\frac{3}{5}}(CL^3)^{\frac{2}{5}}T^{\frac{2}{5}},(\ln d)^{\frac{2}{3}}(L^3+\beta L^2)^{\frac{1}{3}}T^{\frac{1}{3}}}}.
\]
\end{corollary}

\begin{proof}
    Using \pref{thm: OCO}, the alternating regret of Hedge can be bounded by:
    \[
    \RegAlt\le \order\rbr{\frac{B_\psi}{\eta}+CL^3\eta^{3/2}T+\rbr{L^3+\beta L^2}\eta^2 T},
    \]
    where $B_\psi = \max_{u_1,u_2\in\calX}\rbr{\psi(u_1)-\psi(u_2)}$. 
    The maximum value of the negative entropy regularizer is $0$, attained at a pure strategy point, and the minimum value is $-\ln d$, attained by the uniform distribution over the actions. Thus, $B_\psi = \ln d$. Plugging in the value of $\eta$ finishes the proof.
\end{proof}

Once again, for losses with $C=0$, our regret bound is of order $\otil(T^\frac{1}{3})$ (with only logarithmic dependence on $d$ hidden in $\otil(\cdot)$), improving over the $\otil(d^{\frac{2}{3}}T^\frac{1}{3})$ bound of continuous Hedge.

%% file: lower_bound.tex
\section{Alternating Regret Lower Bounds}\label{sec: lower_bound}

While our work significantly advances our understanding on alternating regret upper bounds, 
there are no existing alternating regret lower bounds at all.
Towards closing this gap, we make an initial attempt by considering the special case of the expert problem (OLO over simplex $\Delta_d$) and providing two algorithm-specific lower bounds.
In the first result, 
echoing \pref{thm:hedge_simplex}, 
we show that the worst-case alternating regret of Hedge is $\Omega(T^{\frac{1}{3}})$, 
thereby giving a tight characterization of Hedge's alternating regret (at least for the case with a fixed learning rate).

\begin{theorem}\label{thm:hedge_lower_bound}
   For $d\ge 3$ and any $\eta>0$, there exists a loss sequence $\{\ell_t\}_{t=1}^T$ where $\ell_t\in\sbr{0,1}^d$ for all $t\in[T]$, such that Hedge suffers $\RegAlt=
    \Omega(T^{\frac{1}{3}})$.
\end{theorem}
The full proof is deferred to \pref{app:lower_bound}. To provide a sketch, we consider two environments for $d=3$. In the first environment, the loss vectors are cycling among the three one-hot loss vectors, i.e. $\ell_{3k-2}=(1,0,0),\ell_{3k-1}=(0,1,0),\ell_{3k}=(0,0,1)$ for $k=1,2,\dots,\frac{T}{3}$. We then calculate the closed form of $p_t$, which also turns out to cycle among three vectors. Direct calculation then shows that Hedge achieves an $\Omega(\eta^2 T)$ alternating regret. In the second environment, the loss vectors are all $(1,0,0)$. Via a direct calculation, we show that the alternating regret is $\Omega(\frac{1}{\eta})$. Combining both cases proves 
that the worst case alternating regret is $\Omega(\max\{\eta^2 T, \frac{1}{\eta}\}) = \Omega(T^{\frac{1}{3}})$.

We then move on to another widely-used algorithm called Predictive Regret Matching$^+$ (PRM$^+$) \citep{farina2021faster}. \citet{cai2023last} show that, unlike its other variants, \PRM coupled with alternation has empirically shown fast last-iterate convergence in two-player zero-sum games. However, the reason behind this is not well known. 
One may wonder whether the fast convergence rate of \PRM is due to small alternating regret in the adversarial setting. 
Unfortunately, we show that this is not the case: in the adversarial setting, \PRM actually achieves $\Omega(\sqrt{T})$ alternating regret in the worst case, which is even worse than Hedge. 
In fact, the same is true for another popular algorithm called Optimistic Online Gradient Descent (OOGD) \citep{rakhlin2013optimization}, even though it is well known that in a game setting with simultaneous play, OOGD achieves $o(\sqrt{T})$ regret.
We summarize these lower bounds in the following theorem.
\begin{theorem}\label{thm: PRM+}
    There exists a loss sequence $\{\ell_t\}_{t=1}^T$ where $\ell_t\in\sbr{-4,4}^d$ for all $t\in[T]$, such that $\RegAlt=
    \Omega(\sqrt{T})$ for OOGD and \PRM.
\end{theorem}
The proof is deferred to \pref{app:lower_bound}. Specifically, the loss sequence we create for both algorithms is alternating between $(4,0)$ and $(-2,0)$. Then, we show that for OOGD with learning rate $\eta>0$, when $t\leq \Theta(\frac{1}{\eta})$, the alternating regret is $\Theta(\frac{1}{\eta})$; when $t\geq \Theta(\frac{1}{\eta})$, $p_t$ starts to alternate between $[2\eta,1-2\eta]$ and $[0,1]$, leading to $\Theta(\eta T)$ alternating regret. Combining both cases, the overall alternating regret is $\Omega(\frac{1}{\eta}+\eta T)=\Omega(\sqrt{T})$. For \PRM, we can also show that under this loss sequence, the prediction of \PRM converges to the optimal decision $(0,1)$ at a rate of $1/\sqrt{T}$ approximately, leading to an $\Omega(\sqrt{T})$ alternating regret. 
The implication of this result is that, 
to demystify the practical success of \PRM,
one must take the game structure into account.

%% file: conclusion.tex
\section{Conclusion and Open Problems}

In this work, following up on~\citet{cevher2024alternation} and addressing their main open question, we show that $\order(T^{1/3})$ alternating regret is possible for general Online Convex Optimization via simple algorithms.
The main open problem remains to be: what the minimax optimal alternating regret is for OLO/OCO?
This appears to be highly nontrivial even for very special cases --- for example, even for 1-dimensional OLO over $[-1,1]$ where $\order(\log T)$ alternating regret is achieved by~\citet{cevher2024alternation}, figuring out whether the optimal bound is $\Theta(\log T)$ or $\Theta(1)$ appears to require new techniques.

\paragraph{Acknowledgement}
HL thanks Yang Cai, Gabriele Farina, Julien Grand-Cl{\'e}ment, Christian Kroer, Chung-Wei Lee, and Weiqiang Zheng for discussions related to regret matching with alternation.
He is supported by NSF award IIS-1943607.

%% file: appendix_pre.tex
\section{Omitted Details in \pref{sec:alt-game}}\label{app:alt-game}
In this section, we show the proof of \pref{thm:zero-sum} and \pref{thm:general-sum}, which are proven according to the definition of NE and CCE respectively. 

\begin{proof}[Proof of \pref{thm:zero-sum}]
    Let $y_0$ be the output of $\Alg_y$ with an empty sequence of loss functions. According to \pref{alg:alternation}, the loss function for the \( x \)-player at round \( t \) is given by \( u_1(\cdot, y_t) \), while the loss function for the \( y \)-player at the same round is $ u_2(x_t, \cdot)$.
    By definition of alternating regret and using $u_2=-u_1$, we know that for any $x\in\calX$ and $y\in\calY$,
    \begin{align*}
        &\sum_{t=1}^T\left(u_1(x_t,y_t)+u_1(x_{t+1},y_t)\right) - 2\sum_{t=1}^T u_1(x,y_t)\leq \RegAlt^x,\\
        &\sum_{t=0}^{T-1}\left(u_2(x_{t+1},y_{t})+u_2(x_{t+1},y_{t+1})\right) - 2\sum_{t=0}^{T-1} u_2(x_{t+1},y)\\
        &=2\sum_{t=0}^{T-1} u_1(x_{t+1},y) - \sum_{t=0}^{T-1}\left(u_1(x_{t+1},y_{t})+u_1(x_{t+1},y_{t+1})\right)\leq \RegAlt^y.
    \end{align*}
    Taking a summation over both equations, we obtain that
    \begin{align*}
        2\sum_{t=1}^{T} u_1(x_{t},y) - 2\sum_{t=1}^T u_1(x,y_t)&\leq \RegAlt^x + \RegAlt^y+u_1(x_1,y_0)-u_1(x_{T+1},y_T)\\
        &\leq \RegAlt^x + \RegAlt^y + 2.
    \end{align*}
    Since $u_1(x,y)$ is convex in $x$ and concave in $y$, using Jensen's inequality, we know that for any $x\in\calX$ and $y\in\calY$,
    \begin{align*}
        u_1\left(\overline{x}_T,y\right) - u_1(x,\overline{y}_T) \leq \frac{\RegAlt^x + \RegAlt^y + 2}{2T} = \order\left(\frac{\RegAlt^x + \RegAlt^y}{T}\right),
    \end{align*}
    where $\overline{x}_T=\frac{1}{T}\sum_{t=1}^Tx_t$ and $\overline{y}_T=\frac{1}{T}\sum_{t=1}^Ty_t$. This further means that
    \begin{align*}
        0\leq \max_{y\in\calY}u_1(\overline{x}_T,y)-u_1(\overline{x}_T,\overline{y}_T) \leq \order\left(\frac{\RegAlt^x + \RegAlt^y}{T}\right), \\
        0\leq u_1(\overline{x}_T,\overline{y}_T) - \min_{x\in\calX}u_1(x,\overline{y}_T) \leq \order\left(\frac{\RegAlt^x + \RegAlt^y}{T}\right),
    \end{align*}
    which finishes the proof by the definition of an $\epsilon$-NE with $\epsilon=\order\left(\frac{\RegAlt^x + \RegAlt^y}{T}\right)$.
\end{proof}

\begin{proof}[Proof of \pref{thm:general-sum}]
    Similar to the proof of \pref{thm:zero-sum}, let $y_0$ be the output of $\Alg_y$ with an empty sequence of loss functions. By definition of alternating regret, we know that for any $x'\in\calX$ and $y'\in\calY$,
    \begin{align}
        &\sum_{t=1}^T\left(u_1(x_t,y_t)+u_1(x_{t+1},y_t)\right) - 2\sum_{t=1}^T u_1(x',y_t)\leq \RegAlt^x,\label{eqn:gsum-1}\\
        &\sum_{t=0}^{T-1}\left(u_2(x_{t+1},y_{t})+u_2(x_{t+1},y_{t+1})\right) - 2\sum_{t=0}^{T-1} u_2(x_{t+1},y')\leq \RegAlt^y.\label{eqn:gsum-2}
    \end{align}
    Define $\calP$ to be the uniform distribution over $\{(x_t,y_t),(x_{t+1},y_t)\}_{t\in[T]}$.
    Adding $u_2(x_{T+1},y_T)-u_2(x_1,y_0)+u_2(x_1,y')-u_2(x_{T+1},y')\in[-4,4]$ on both sides of \pref{eqn:gsum-2} and dividing $2T$ on both sides of \pref{eqn:gsum-1} and \pref{eqn:gsum-2}, we know that for any $x'\in\calX$ and $y'\in\calY$,
    \begin{align*}
    \E_{(x,y)\sim \calP}\left[u_1(x,y)\right] - \E_{(x,y)\sim \calP}\left[u_1(x',y)\right]&\leq \frac{\RegAlt^y}{2T},\\
        \E_{(x,y)\sim \calP}\left[u_2(x,y)\right] - \E_{(x,y)\sim \calP}\left[u_2(x,y')\right]&\leq \frac{\RegAlt^y + 4}{2T}.
    \end{align*}
    This finishes the proof by the definition of an $\epsilon$-CCE with $\epsilon=\order\left(\frac{\max\{\RegAlt^x, \RegAlt^y\}}{T}\right)$.
\end{proof}

%% file: appendix_general_oco.tex
\section{Omitted Details in \pref{sec: alt-oco}}\label{app:CEW}
In this section, we show the omitted proofs in \pref{sec: alt-oco}. In the following, we provide the proof of \pref{thm:hedge_simplex}.

\begin{proof}[Proof of \pref{thm:hedge_simplex}]
    We aim to show that
    \begin{align}
        \inner{p_t-u,\ell_t} &= \frac{1}{\eta}\left(\KL(u,p_t) - \KL(u,p_{t+1}) + \KL(p_t,p_{t+1})\right),\label{eqn:regular-expert}\\
        \inner{p_{t+1}-u, \ell_t} &= \frac{1}{\eta}\left(\KL(u,p_t) - \KL(u,p_{t+1}) - \KL(p_{t+1},p_{t})\right).\label{eqn:cheat-expert}
    \end{align}
    Note that the dynamic of Hedge implies that $p_{t+1,i}=\frac{p_{t,i}\exp(-\eta\ell_{t,i})}{\sum_{j=1}^dp_{t,j}\exp(-\eta\ell_{t,j})}$. Direct calculation shows that
    \begin{align*}
        &\frac{1}{\eta}\left(\KL(u,p_t) - \KL(u,p_{t+1}) + \KL(p_t,p_{t+1})\right)
        \\
        &= \frac{1}{\eta}\sum_{i=1}^d (p_{t,i}-u_i)\cdot\log\frac{p_{t,i}}{p_{t+1,i}} \\
        &= \sum_{i=1}^d (p_{t,i}-u_i)\left(\ell_{t,i} + \frac{1}{\eta}\log\left(\sum_{j=1}^dp_{t,j}\exp(-\eta\ell_{t,j})\right)\right) \\
        &=\inner{p_t-u,\ell_t},\\
        &\frac{1}{\eta}\left(\KL(u,p_t) - \KL(u,p_{t+1}) - \KL(p_{t+1},p_{t})\right)
        \\
        &= \frac{1}{\eta}\sum_{i=1}^d (p_{t+1,i}-u_i)\cdot\log\frac{p_{t,i}}{p_{t+1,i}} \\
        &= \sum_{i=1}^d (p_{t+1,i}-u_i)\left(\ell_{t,i} + \frac{1}{\eta}\log\left(\sum_{j=1}^dp_{t,j}\exp(-\eta\ell_{t,j})\right)\right) \\
        &=\inner{p_{t+1}-u,\ell_t}.
    \end{align*}
    Combining \pref{eqn:regular-expert} and \pref{eqn:cheat-expert} for all $t\in[T]$, we obtain that
    \begin{align}\label{eqn:alt-regret-simplex}
        \RegAlt(u) = \frac{2(\KL(u,p_1)-\KL(u,p_{T+1}))}{\eta} + \frac{1}{\eta}\sum_{t=1}^T\left(\KL(p_t,p_{t+1})- \KL(p_{t+1},p_t)\right).
    \end{align}
    Since $p_1$ is the uniform distribution over all $d$ experts, we can bound $\KL(u,p_1)-\KL(u,p_{T+1})\leq \log d$. To control the term $\KL(p_t,p_{t+1})- \KL(p_{t+1},p_t)$, define $G(x)=\log(\sum_{i=1}^d\exp(x_i))$ for $x\in\R^d$ and $L_t=\sum_{\tau\leq t}\ell_{\tau}$. Direct calculation shows that
    \begin{align*}
        &D_G(-\eta L_t, -\eta L_{t-1}) - D_G(-\eta L_{t-1}, -\eta L_{t}) \\
        &=2\log\left(\sum_{i=1}^d\exp(-\eta L_{t,i})\right) - 2\log\left(\sum_{i=1}^d\exp(-\eta L_{t-1,i})\right) + \sum_{i=1}^dp_{t+1,i}\cdot (\eta\ell_{t,i}) + \sum_{i=1}^dp_{t,i}\cdot(\eta\ell_{t,i}) \\
        &=\KL(p_t,p_{t+1})- \KL(p_{t+1},p_t).
    \end{align*}
    
    Applying Lemma A.2 of~\citet{wibisono2022alternating}, we know that $\KL(p_t,p_{t+1})- \KL(p_{t+1},p_t)\leq \frac{4}{3}\|\eta\ell_t\|_{\infty}^3\leq \frac{4}{3}\eta^3$, where the first inequality is because $G(x)$ is $8$-smooth of order 3 as proven in Example A.3~\citet{wibisono2022alternating}. Plugging the above to \pref{eqn:alt-regret-simplex} and picking the optimal $\eta$ finishes the proof.
\end{proof}

In the following, we show the proof for \pref{thm:hedge-cont}.
\begin{proof}[Proof of \pref{thm:hedge-cont}]
    By definition, we know that $p_t(x)=\frac{1}{Z_t}\exp(-\eta F_{t-1}(x))$ and $p_{t+1}(x)=\frac{1}{Z_{t+1}}\exp(-\eta F_{t}(x))$, where $Z_t = \int_{x\in\calX}\exp(-\eta F_{t-1}(x))dx$ and $F_t(x) = \sum_{\tau\leq t}f_{\tau}(x)$.
    Then, direct calculation shows that
    \begin{align*}
        \frac{1}{\eta}\KL(p_t,p_{t+1})        &= \frac{1}{\eta}\int_{x\in\calX}p_t(x)\log\frac{p_t(x)}{p_{t+1}(x)}dx\\
        &= \int_{x\in\calX}p_t(x)\left(f_t(x)+\frac{1}{\eta}\log\frac{Z_{t+1}}{Z_{t}}\right)dx
        \\
        &= \int_{x\in\calX}p_t(x)f_t(x)dx + \frac{1}{\eta}\left(\log Z_{t+1} - \log Z_t\right)
        \\
        \frac{1}{\eta}\KL(p_{t+1},p_{t})&= \frac{1}{\eta}\int_{x\in\calX}p_{t+1}(x)\log\frac{p_{t+1}(x)}{p_{t}(x)}dx\\
        &= \int_{x\in\calX}p_{t+1}(x)\left(-f_t(x)+\frac{1}{\eta}\log\frac{Z_{t}}{Z_{t+1}}\right)dx
        \\
        &= -\int_{x\in\calX}p_{t+1}(x)f_t(x)dx +\frac{1}{\eta}( \log Z_{t} - \log Z_{t+1})
    \end{align*}
    Combining the above two inequalities, we have
   
    \begin{align}
        \int_{x\in\calX}(p_t(x)+p_{t+1}(x))f_t(x)dx + \frac{2}{\eta}(\log Z_{T+1} - \log Z_{1}) = \frac{1}{\eta}\left(\KL(p_t,p_{t+1}) - \KL(p_{t+1},p_t)\right) .\label{eqn:alt-regret-cont}
    \end{align}
    According to the proof of Theorem 3.1 in \citet{bubeck2011introduction}, we know that 
    \begin{align*}
        \log Z_{T+1} - \log Z_{1} \geq -d\log\frac{1}{\gamma} - \eta \sum_{t=1}^Tf_t(u) - \eta T\gamma,
    \end{align*}
    for any $u\in\calX$ and $\gamma\in [0,1]$. Picking $\gamma = \frac{1}{T}$ and using the fact that $\int_{x\in\calX}p_t(x)f_t(x)dx\geq f_t(x_t)$ and $\int_{x\in\calX}p_{t+1}(x)f_t(x)dx\geq f_t(x_{t+1})$ due to convexity, we know that
\begin{align}\label{eqn:reg-cont-main}
    \RegAlt \leq \frac{d\log T}{\eta} + \frac{1}{\eta}\sum_{t=1}^T\left(\KL(p_t,p_{t+1})-\KL(p_{t+1},p_t)\right).
\end{align}
Next, we bound $\KL(p_t,p_{t+1})-\KL(p_{t+1},p_t)$. Define $H(F) = \log\int_{x\in\calX}\exp(-\eta F(x))dx$.
By definition of $\KL(p,q)$, we know that
\begin{align}\label{eqn:KL}
    \KL(p_t,p_{t+1}) = \int_{x\in\calX}p_t(x)\log\frac{p_t(x)}{p_{t+1}(x)}dx = H(F_{t+1}) - H(F_{t}) + \eta\cdot \E_{x\sim p_t}\left[f_t(x)\right].
\end{align}
Define $p_{t,s}$ to be the distribution where $p_{t,s}(x)\propto \exp(-\eta ((1-s) F_{t-1}(x) +sF_{t}(x)))$ and $Z_{t,s}=\int_{x\in\calX}\exp(-\eta ((1-s)F_{t-1}(x) +sF_{t}(x)))dx$. By definition, we have $p_{t,0}=p_t$, $p_{t,1}=p_{t+1}$, $Z_{t,0}=Z_t$ and $Z_{t,1}=Z_{t+1}$. Direct calculation shows that
\begin{align}
    \nabla_s\log Z_{t,s} &= \frac{\nabla_s Z_{t,s}}{Z_{t,s}} = \frac{1}{Z_{t,s}}\int_{x\in\calX}-\eta f_t(x)\exp(-\eta((1-s)F_{t-1}(x)+sF_t(x)))dx \nonumber \\
    &= -\E_{x\sim p_{t,s}}[\eta f_t(x)], \label{eqn:mean_exp}\\
    \nabla_s^2\log Z_{t,s} &= \frac{\int_{x\in\calX}\eta^2f_t(x)^2\exp(-\eta((1-s)F_{t-1}(x)+sF_t(x)))dx}{Z_{t,s}} \nonumber \\
    &\qquad - \left(\frac{\int_{x\in\calX}\eta f_t(x)\exp(-\eta((1-s)F_{t-1}(x)+sF_t(x)))dx}{Z_{t,s}}\right)^2 \nonumber\\
    &=\Var_{x\sim p_{t,s}}[\eta f_t(x)]. \label{eqn:var_exp}
\end{align}

According to Fundamental theorem of calculus, we know that
\begin{align*}
    \log Z_{t,1} &= \log Z_{t,0} + \int_0^1\nabla_s \log Z_{t,s}ds \\
    &=\log Z_{t,0} + [\nabla_s \log Z_{t,s}]_{s=0} +\int_0^1\int_0^\alpha \nabla_s^2 \log Z_{t,s}dsd\alpha \\
    &= \log Z_{t,0} + [\nabla_s \log Z_{t,s}]_{s=0} +\int_0^1\left(\int_s^1d\alpha\right) \nabla_s^2 \log Z_{t,s}ds \\
    &= \log Z_{t,0} + [\nabla_s \log Z_{t,s}]_{s=0} +\int_0^1(1-s)\nabla_s^2 \log Z_{t,s}ds.
\end{align*}

Then, based on \pref{eqn:KL}, we can obtain that
\begin{align*}
     \KL(p_t,p_{t+1}) &= \log Z_{t,1} - \log Z_{t,0} - \left[\nabla_s \log Z_{t,s}\right]_{s=0} \\
     &= \int_0^1 (1-s)[\nabla_{s}^2\log Z_{t,s}]ds \\
     &= \int_0^1 (1-s)\Var_{x\sim p_{t,s}}[\eta f_t(x)]ds,
\end{align*}
where the third equality is due to \pref{eqn:var_exp}. Similarly, we can obtain that
\begin{align*}
     \KL(p_{t+1},p_{t}) &= \log Z_{t,0} - \log Z_{t,1} - \left[\nabla_s \log Z_{t,s}\right]_{s=1} \\
     &= \int_0^1 (1-s)\Var_{x\sim p_{t,1-s}}[-\eta f_t(x)]ds \\
     &= \int_0^1 s'\Var_{x\sim p_{t,s'}}[\eta f_t(x)]ds', 
\end{align*}
where the last equality is by a variable replacement $s'=1-s$ and $\Var[x]=\Var[-x]$.

Combining the above two equalities, we know that
\begin{align}\label{eqn:KL-1}
    \KL(p_t,p_{t+1}) - \KL(p_{t+1},p_t) = \int_0^1(1-2s)\Var_{x\sim p_{t,s}}[\eta f_t(x)]ds.
\end{align}

To further analyze $\int_0^1(1-2s)\Var_{x\sim p_{t,s}}[\eta f_t(x)]ds$, let $v = s-\frac{1}{2}$ and $p_{t,v}\propto \exp(-(\frac{1}{2}-v)\eta F_t(x)-(\frac{1}{2}+v)\eta F_{t+1}(x))$ with an abuse of notation. Then, we have
\begin{align}
    &\frac{1}{2}\int_0^1(1-2s)\Var_{x\sim p_{t,s}}[\eta f_t(x)]ds \nonumber\\
    &= -\int_{-\frac{1}{2}}^{\frac{1}{2}}v\cdot\Var_{x\sim p_{t,v}}[\eta f_t(x)]dv \nonumber\\
    &= -\int_{0}^{\frac{1}{2}}v\cdot\Var_{x\sim p_{t,v}}[\eta f_t(x)]dv -\int_{-\frac{1}{2}}^{0}v\cdot\Var_{x\sim p_{t,v}}[\eta f_t(x)]dv \nonumber \\
    &= \int_0^{\frac{1}{2}}v\left(\Var_{x\sim p_{t,-v}}[\eta f_t(x)] - \Var_{x\sim p_{t,v}}[\eta f_t(x)]\right)dv. \label{eqn:KL-2}
\end{align}

We then show how to bound $\Var_{x\sim p_{t,-v}}[\eta f_t(x)] - \Var_{x\sim p_{t,v}}[\eta f_t(x)]$. For notational convenience, let $p_{t,-v}=q_1$ and $p_{t,v}=q_2$.

Direct calculation shows that
\begin{align}
    &\Var_{x\sim q_1}[\eta f_t(x)] - \Var_{x\sim q_2}[\eta f_t(x)] \nonumber\\
    &=\eta^2\rbr{\E_{q_1}[f_t(x)^2] - \E_{q_2}[f_t(x)^2]} - \eta^2(\E_{q_1}[f_t(x)] + \E_{q_2}[f_t(x)])(\E_{q_1}[f_t(x)] - \E_{q_2}[f_t(x)]) \nonumber\\
    &\leq \eta^2\max_{x\in\calX}f_t^2(x) \|q_1-q_2\|_{TV} + 2\eta^2\max_{x\in\calX}f_t^2(x)\|q_1-q_2\|_{TV}\nonumber\\
    &\leq 3\eta^2 \|q_1-q_2\|_{TV}, \label{eqn:KL-3}
\end{align}
where we use $f_t(x)\in[-1,1]$. To further bound $\|q_1-q_2\|_{TV}$, we define $G(x)=-\frac{1}{2}(F_t(x)+F_{t+1}(x))$ for notational convenience. By definition of $q_1$ and $q_2$, we know that for any $x\in\calX$
\begin{align*}
    \frac{q_1(x)}{q_2(x)} &= \frac{\int_{x'}\exp(G(x')+\eta vf_t(x'))dx'}{\int_{x'}\exp(G(x')-\eta vf_t(x'))dx'} \cdot \frac{\exp(G(x)-\eta vf_t(x))}{\exp(G(x)+\eta vf_t(x))} \\
    &\leq \frac{\int_{x'}\exp(G(x')-\eta vf_t(x'))dx'}{\int_{x'}\exp(G(x')-\eta vf_t(x'))dx'} \cdot \frac{\exp(G(x)+vf_t(x))}{\exp(G(x)+vf_t(x))}\exp(\eta) \\
    &= \exp(\eta),
\end{align*}
where the inequality is because $v\leq \frac{1}{2}$ and $f_t(x)\in[-1,1]$.
Similarly, we can also show that
\begin{align*}
    \frac{q_1(x)}{q_2(x)} &= \frac{\int_{x'}\exp(G(x')+\eta vf_t(x'))dx'}{\int_{x'}\exp(G(x')-\eta vf_t(x'))dx'} \cdot \frac{\exp(G(x)-\eta vf_t(x))}{\exp(G(x)+\eta vf_t(x))} \\
    &\geq \frac{\int_{x'}\exp(G(x')-\eta vf_t(x'))dx'}{\int_{x'}\exp(G(x')-\eta vf_t(x'))dx'} \cdot \frac{\exp(G(x)+vf_t(x))}{\exp(G(x)+vf_t(x))}\exp(-\eta) \\
    &= \exp(-\eta).
\end{align*}
Therefore, we know that $\|q_1-q_2\|_{TV}\leq \exp(\eta) - 1\leq 2\eta $ since $\eta\leq 1$. Combining \pref{eqn:KL-1}, \pref{eqn:KL-2}, and \pref{eqn:KL-3}, we conclude that
\begin{align}\label{eqn:KL-diff-cont}
    \KL(p_t,p_{t+1}) - \KL(p_{t+1},p_{t}) = \order(\eta^3). 
\end{align}
Combining \pref{eqn:reg-cont-main} and \pref{eqn:KL-diff-cont} and picking $\eta$ optimally finishes the proof.

\end{proof}

%% file: appendix_scb.tex
\section{Preliminary for Self-Concordant Barrier}\label{app:scb}
\begin{definition}\label{def:scb}
    Let $\psi:\intO\to \fR$ be a $C^3$-smooth convex function. $\psi$ is called a self-concordant barrier on $\Omega$ if it satisfies: 
    \begin{itemize}
    \item $\psi(x_i)\to \infty$ as $i\to \infty$ for any sequence $x_1,x_2,\dots\in\intO\subset \fR^d$ converging to the boundary of $\Omega$;
    \item for all $w\in\intO$ and $h\in\fR^d$, the following inequality always holds:
    $$
    \sum^d_{i=1}\sum^d_{j=1}\sum^d_{k=1}\frac{\partial^3\psi(w)}{\partial w_i\partial w_j\partial w_k}h_ih_jh_k\le 2\|h\|_{\nabla^2\psi(w)}^3.
    $$
    
    \end{itemize}
    We further call $\psi$ is a $\nu$-self-concordant barrier if it satisfies the conditions above and also
    $$\inner{\nabla\psi(w),h}\le \sqrt{\nu}\|h\|_{\nabla^2\psi(w)}$$
    for all $w\in\intO$ and $h\in\fR^d$.
\end{definition}

To deal with the infinite range of the barrier $\psi$ on $\Omega$, we define a shrunk version of the decision domain $\Omega$ using Minkowsky functions.

\begin{definition}\label{def:minkowsky-functions}
    Define the Minkowsky function $\pi_w:\Omega\mapsto\fR$ associated with a point $w\in\intO$ and $\Omega$ as:
    \begin{equation*}
        \pi_w(u) = \inf\cbr{s>0~\Big|~w+\frac{u-w}{s}\in\Omega}.
    \end{equation*}
\end{definition}

\begin{lemma}[Proposition 2.3.2 in \citep{nesterov1994interior}]\label{lem:self_concordant}
    Let $\psi$ be a $\nu$-self-concordant barrier on $\Omega\in\fR^d$. Then, for any $u,w\in\intO$, we have
    \begin{equation*}
        \psi(u)-\psi(w)\le\nu\ln\rbr{\frac{1}{1-\pi_w(u)}}.
    \end{equation*}
    This means that for all $u\in\Omega'\triangleq\cbr{(1-\epsilon)x+\epsilon w:x\in\Omega}$ where $\epsilon>0$, we have
    \begin{equation*}
        \psi(u)-\psi(w)\le\nu\ln\rbr{1/\epsilon},
    \end{equation*}
    since $\pi_{w}(u)\leq 1-\epsilon$.
\end{lemma}

%% file: appendix_alt_oco.tex
\section{Omitted Details in \pref{sec:SC}}\label{app: alt-oco}

\begin{lemma}[Conjugate Duality]\label{lem:conj_dual}
    Suppose that $\psi:\R^d\mapsto\R$ is differentiable and strictly convex. Let $\psi^*(w)=\sup_{x\in\R^d}\left(\inner{x,w}-\psi(x)\right)$. Then, for any $x_1,x_2\in\R^d$, we have $D_{\psi}(x_1,x_2) = D_{\psi^*}(\nabla\psi(x_2),\nabla\psi(x_1))$.
\end{lemma}
\begin{proof}
    Direct calculation shows that
    \begin{align*}
        D_{\psi}(x_1,x_2) &= \psi(x_1) - \psi(x_2) - \inner{\nabla\psi(x_2),x_1-x_2} \\
        &= \psi(x_1) - \psi(x_2) - \inner{\nabla\psi(x_2),x_1-x_2} + \inner{x_1,\nabla\psi(x_1)} - \inner{x_1,\nabla\psi(x_1)}\\
        &=(\inner{\nabla\psi(x_2),x_2}-\psi(x_2)) - (\inner{\nabla\psi(x_1),x_1}-\psi(x_1)) - \inner{x_1,\nabla\psi(x_2) - \nabla\psi(x_1)} \\
        &=D_{\psi^*}(\nabla\psi(x_2),\nabla\psi(x_1)),
    \end{align*}
    where the last equality is because $\nabla\psi^*(\nabla\psi(x))=x$ and $\psi(x)+\psi^*(\nabla\psi(x))=\inner{x,\nabla\psi(x)}$.
\end{proof}

\begin{lemma}[Stability of FTRL]\label{lem:path_length}
    \pref{alg:OCO} ensures
    $
        \norm{x_t-x_{t+1}}_2 \le \frac{\eta}{\sigma}\norm{\nabla f_t(x_t)}_2,
    $
    for all $t$.
\end{lemma}

\begin{proof}
    Let $F_{t+1}(x)\triangleq\sum_{\tau=1}^tf_\tau(x)+\frac{1}{\eta}\psi(x)$. Since $x_{t+1}$ is the minimizer of $F_{t+1}$ in $\calX$, using first-order optimality, we have $F_{t+1}(x_{t+1})\le F_{t+1}(x)-D_{F_{t+1}}(x,x_{t+1})$ for all $x\in\calX$. Picking $x=x_{t}$, we get
    \begin{equation}\label{eqn:bregman_1}
        F_{t+1}(x_{t+1}) - F_t(x_t) \le f_t(x_t) - D_{F_{t+1}}(x_t,x_{t+1}).
    \end{equation}
    Similarly, for $F_t$, we have
    \begin{equation}\label{eqn:bregman_2}
        F_t(x_t) - F_{t+1}(x_{t+1}) \le -f_t(x_{t+1}) - D_{F_t}(x_{t+1},x_t).
    \end{equation}
    Summing up \pref{eqn:bregman_1} and \pref{eqn:bregman_2}, we can obtain that
    \begin{align*}
        f_t(x_t)-f_t(x_{t+1}) \ge D_{F_{t+1}}(x_t,x_{t+1}) + D_{F_t}(x_{t+1},x_t)\ge \frac{\sigma}{\eta}\norm{x_t-x_{t+1}}_2^2,
    \end{align*}
    where the last inequality is because $\psi$ is $\sigma$-strongly convex with respect to $\|\cdot\|_2$ within domain $\calX$, implying that both $F_t$ and $F_{t+1}$ are $\frac{\sigma}{\eta}$-strongly convex.
    Using the convexity of $f_t$, we further have $f_t(x_t)-f_t(x_{t+1})\le \innerp{\nabla f_t(x_t)}{x_t-x_{t+1}}\le \norm{\nabla f_t(x_t)}_2\cdot \norm{x_t-x_{t+1}}_2$. Combining these two inequalities, we get $\norm{x_t-x_{t+1}}_2 \le \frac{\eta}{\sigma}\norm{\nabla f_t(x_t)}_2$.
    \end{proof}

The rest of the section includes the omitted proofs for several lemmas.
\begin{proof}[Proof of \pref{lem:entropic}]
    Given any $\theta\in\R^d$, define $p_{\theta}(x) = \exp\rbr{\innerp{\theta}{x}-\psi^*(x)}\mathbbm{1}\cbr{x\in\calX}$ to be the exponential distribution over $\calX$ with parameter $\theta$.
Then, according to Lemma 1 of \citet{bubeck2014entropic}, we know that
\begin{align*}
    &\nabla^3\psi^*(\theta) 
    = \E_{X\sim p_\theta}\sbr{\rbr{X-x(\theta)}\otimes\rbr{X-x(\theta)}\otimes\rbr{X-x(\theta)}},\\
    &\nabla^2\psi(x)=\rbr{\E_{X\sim p_\theta}\sbr{\rbr{X-x}\rbr{X-x}^\T}}^{-1},
\end{align*}
where $x(\theta)=\E_{X\sim p_{\theta}}\sbr{X}$. 
Therefore, for any $h\in\R^d$ such that $\|h\|_2\leq 1$, we know that
\begin{align*}
    \sum_{i=1}^d\sum^d_{j=1}\sum^d_{k=1}\frac{\partial^3\psi^*(\theta)}{\partial \theta_i\partial \theta_j\partial \theta_k}h_ih_jh_k \leq \max_{x,y\in\calX}|\inner{h,x-y}|^3\leq D^3,
\end{align*}
where $D$ is the diameter of $\calX$. \\In addition, since the maximum eigenvalue of $\nabla^{-2}\psi(x)=\E_{X\sim p_\theta}\sbr{\rbr{X-x}\rbr{X-x}^\T}$ is upper bounded by $D^2$, we know that $\nabla^2\psi(x)$ has minimum eigenvalue lower bounded by $1/D^2$. Combining both arguments, we see that $\psi$ satisfies \pref{asp: psi} with $M=D^3$ and $\sigma=1/D^2$. 
\end{proof}

\begin{proof}[Proof of \pref{lem:smoothness of ball case}]
    We can see that $\psi(x)=-\ln\rbr{1-\norm{x}_2^2}$ is a Legendre function. Now, we show that it is $2$-strongly convex with respect to the $\ell_2$-norm.
    \begin{align*}
        \nabla\psi(x) &= \frac{2x}{1-\norm{x}_2^2},\\
        \nabla^2\psi(x) &= \frac{2\mathbf{I}}{1-\norm{x}_2^2}+\frac{4xx^\T}{\rbr{1-\norm{x}_2^2}^2},
    \end{align*}
    where $\mathbf{I}$ represents the $d$-dimensional identity matrix.
    Thus, for any $h\in\fR^d$ and $\norm{h}_2=1$, we have
    \begin{align*}
        \nabla^2\psi(x)[h,h] = \frac{2\norm{h}_2^2}{1-\norm{x}_2^2}+\frac{4\rbr{\innerp{x}{h}}^2}{\rbr{1-\norm{x}_2^2}^2}\ge 2.
    \end{align*}
    It remains to show that $\psi^*$ is $4$-smooth of order $3$ with respect to the $\ell_2$-norm. Below, we calculate the derivatives of $\psi^*$:
    \begin{align*}
    \nabla \psi^*(w)
    &=\frac{w}{1+\sqrt{1+\norm{w}_2^2}},\\
    \nabla^2 \psi^*(w)
    &=\frac{I}{1+\sqrt{1+\norm{w}_2^2}}-\frac{ww^\top}{\rbr{1+\sqrt{1+\norm{w}_2^2}}^2\cdot \sqrt{1+\norm{w}_2^2}}.
    \end{align*}
    Denoting $g(w)=\sqrt{1+\norm{w}_2^2}$, we have
    \begin{align*}
    \nabla^2 \psi^*(w)
    =\frac{\mathbf{I}}{1+g(w)}-\frac{ww^\top}{\rbr{1+g(w)}^2\cdot g(w)}.
    \end{align*}
    Then, we have
    \begin{align*}
    \nabla^3 \psi^*(w)&= -\frac{\textrm{Sym}(\mathbf{I}\otimes w)}{g(w)\rbr{1+g(w)}^2} 
    + \frac{\rbr{1+3g(w)}w\otimes w\otimes w}{\rbr{1+g(w)}^3 g^3(w)}.
    \end{align*}
    Here, $\textrm{Sym}(\mathbf{I}\otimes w)$ denotes symmetrization over all permutations of the indices of the tensor product $I\otimes w$. 
    Then, we check the condition in \pref{def:3rd-order smoothness} for any $h\in\fR^d$ and $\norm{h}_2=1$.
    \begin{align*}
    &\sum_{i=1}^d\sum^d_{j=1}\sum^d_{k=1}\frac{\partial^3\psi(w)}{\partial w_i\partial w_j\partial w_k}h_ih_jh_k \\& =\sum_{i=1}^d\sum^d_{j=1}\sum^d_{k=1} -\frac{\rbr{\delta_{ij}w_k+\delta_{ik}w_j+\delta_{jk}w_i}h_ih_jh_k}{g(w)\rbr{1+g(w)}^2}+\frac{\rbr{1+3g(w)}w_ih_i w_jh_j w_kh_k}{g^3(w)\rbr{1+g(w)}^3}\\
    &= -\frac{3\sum_{i=1}^d w_ih_i\cdot \sum_{k=1}^d h_k^2 }{g(w)\rbr{1+g(w)}^2}+\frac{\rbr{1+3g(w)}\rbr{\sum_{i=1}^d w_ih_i}^3}{g^3(w)\rbr{1+g(w)}^3}.
    \end{align*}
    By Cauchy-Schwarz inequality, we have $\abs{\innerp{w}{h}}\le \norm{w}_2\cdot  \norm{h}_2$. Combining with the fact that $\norm{h}_2=1$,
    it holds that
    \begin{align*}
    &\sum_{i=1}^d\sum^d_{j=1}\sum^d_{k=1}\frac{\partial^3\psi(w)}{\partial w_i\partial w_j\partial w_k}h_ih_jh_k\\ &\le \frac{3 \norm{w}_2 }{g(w)\rbr{1+g(w)}^2}+\frac{\rbr{1+3g(w)}{\norm{w}_2^3}}{g^3(w)\rbr{1+g(w)}^3}.
    \end{align*}
    Recall that $g(w)=\sqrt{1+\norm{w}_2^2}$, it holds that $g(w)\ge \max(1,\norm{w}_2)$. Therefore, we have
    \begin{align*}
    \sum_{i=1}^d\sum^d_{j=1}\sum^d_{k=1}\frac{\partial^3\psi(w)}{\partial w_i\partial w_j\partial w_k}h_ih_jh_k\le 4.
    \end{align*}
\end{proof}
\begin{proof}[proof of \pref{lem:negative entropy}]
    Direct calculation shows that the gradient and Hessian of $
        \psi(x)$ are as follows:
    \begin{align*}
    \frac{\partial \psi}{\partial x_j}&=\log x_j+1-\log(1-\sum_{i=1}^{d-1}x_i)-1=\log x_j-\log(1-\sum_{i=1}^{d-1}x_i),\\
    \nabla^2\psi(x)&=\begin{bmatrix}
    \frac{1}{x_1}&&&\\
    &\frac{1}{x_2}&&\\
    &&\cdots&\\
    &&&\frac{1}{x_{d-1}}
    \end{bmatrix}+\frac{1}{1-\sum_{i=1}^{d-1}x_i}\cdot 
    \begin{bmatrix}
    1&1&\cdots&1\\
    1&1&\cdots&1\\
    \vdots&\vdots&\vdots&\vdots\\
    1&1&\cdots&1
    \end{bmatrix}\succeq \begin{bmatrix}
    \frac{1}{x_1}&&&\\
    &\frac{1}{x_2}&&\\
    &&\cdots&\\
    &&&\frac{1}{x_{d-1}}
    \end{bmatrix}.
    \end{align*}
    Since $x_i\leq 1$ for all $i\in[d-1]$, we have $\nabla^2\psi(x)\succeq I$, meaning that $\psi(x)$ is $1$-strongly convex with respect to $\|\cdot\|_2$.    

    Direct calculation also shows that $\psi^*(w)=\log\left(1+\sum_{i=1}^{d-1}\exp(w_i)\right)$ is the convex conjugate of $\psi(x)$. Define a distribution $p_w$ over $[d]$ where $p_w(i)=\exp(w_i-\psi^*(w))$ for $i\in [d-1]$ and $p_w(d)=\exp(-\psi^*(w))$. For any $v\in\R^{d-1}$, $\nabla^3\psi^*(w)[v,v,v]$ can be calculated as follows:
    \begin{align*}
        \nabla^3\psi^*(w)[v,v,v] = \E_{I\sim p_w}[(v_I-\bar{v})^3],
    \end{align*}
    where $v_d=0$ and $\bar{v}=\sum_{i=1}^{d-1}v_i\cdot p_w(i)$.
    Therefore, when $\|v\|_2=1$, we have $|v_I-\bar{v}|\leq 2$ for all $I\in[d]$ and 
    \[
    |\nabla^3\psi^*(w)[v,v,v]|\le \mathbb{E}_{I\sim p_w}[|v_I-\Bar{v}|^3]\le  2^3=8.
    \]
    This implies that $\psi^*$ is $8$-smooth of order $3$ with respect to $\ell_2$-norm.
\end{proof}

\begin{lemma}\label{lem:3rd-order-smooth-F}
    Let $F_t(x)\triangleq \sum_{\tau=1}^t f_\tau(x)+\frac{1}{\eta}\psi(x)$. If $\psi$ is Legendre, $\sigma$-strongly convex and its convex conjugate $\psi^*$ is $M$-smooth of order $3$ with respect to $\|\cdot\|_2$, and $f_\tau$ is $C$-self-concordant for all $\tau\in[t]$, then $F_t^*$ is $\rbr{\frac{2C\eta^{3/2}}{\sigma^{3/2}}+M\eta^2}$-smooth of order $3$ with respect to $\|\cdot\|_2$.
\end{lemma}

\begin{proof}
    By definition of convex conjugate, we know that
    \begin{align*}
            \nabla F_t^*(y) &= \rbr{\nabla F_t}^{-1}(y),\\
    \nabla^2 F_t^*(y) &= \rbr{\nabla^2 F_t(x)}^{-1},
    \end{align*}
    where $x = \rbr{\nabla F_t}^{-1}(y)$, the dual variable of $y$.
    To obtain the third derivative, we differentiate both sides of the equation above and obtain that:
    \[
    \nabla^3 F_t^*(y)[h, h, h] = - \nabla^3 F_t(x) \left[ (\nabla^2 F_t(x))^{-1} h, (\nabla^2 F_t(x))^{-1} h, (\nabla^2 F_t(x))^{-1} h \right], \forall~h\in\fR^d.
    \]
    We want to bound $\abs{\nabla^3 F_t^*(y)[h, h, h]}$ for all $h$. Taking $A=\frac{1}{\eta}\nabla^2 \psi(x)$ and $B=\nabla^2 \rbr{\sum_{\tau=1}^tf_\tau(x)}$, we have
    \begin{align*}
    &-\nabla^3 \frac{\psi(x)}{\eta}\left[(\nabla^2 F_t(x))^{-1} h, (\nabla^2 F_t(x))^{-1} h, (\nabla^2 F_t(x))^{-1} h\right] \\
    &= -\nabla^3 \frac{\psi(x)}{\eta}\left[A^{-1}A(A+B)^{-1} h, A^{-1}A(A+B)^{-1} h, A^{-1}A(A+B)^{-1} h\right] \\
    &= \nabla^3\rbr{\frac{\psi}{\eta}}^*(y)[A(A+B)^{-1}h,A(A+B)^{-1}h,A(A+B)^{-1}h] \\
    &\leq \eta^2 M\|A(A+B)^{-1}h\|_2^3\tag{by \pref{lem:psi-by-eta-smoothness}}\\&\leq \eta^2M\norm{h}_2^3.\tag{by \pref{lem:max_singular_val}}
    \end{align*}

    Since for all $\tau\in[t]$, $f_\tau$ is $C$-self-concordant, we have
    \begin{align*}
        &-\sum_{\tau=1}^t\nabla^3 f_\tau(x)\sbr{(\nabla^2 F_t(x))^{-1} h, (\nabla^2 F_t(x))^{-1} h, (\nabla^2 F_t(x))^{-1} h} \\
        &\le \sum_{\tau=1}^t2C\rbr{\nabla^2 f_\tau(x)\sbr{(\nabla^2 F_t(x))^{-1} h, (\nabla^2 F_t(x))^{-1} h}}^{3/2}\tag{by \pref{def:self-concordance}}\\
        &\le 2C\rbr{\sum_{\tau=1}^t\nabla^2f_\tau(x)\sbr{(\nabla^2 F_t(x))^{-1} h, (\nabla^2 F_t(x))^{-1} h}}^{3/2} \tag{since $\nabla^2f_\tau(x)$ is PSD}\\
        &= 2C\rbr{B\sbr{(A+B)^{-1}h,(A+B)^{-1}h}}^{3/2}\\
        &\le 2C\rbr{(A+B)\sbr{(A+B)^{-1}h,(A+B)^{-1}h}}^{3/2} \tag{since $A$ is PSD}\\
        &\le 2C\rbr{\frac{\eta}{\sigma}\norm{h}_2^2}^{3/2}\tag{since $A\succeq \frac{\sigma}{\eta} I$}\\
        &\le \frac{2C\eta^{3/2}}{\sigma^{3/2}} \norm{h}_2^3.
    \end{align*}
    Now we are ready to estimate the third-order smoothness of $F_t^*$. Expanding the third derivative of $F_t$ in terms of the derivatives of $f_\tau$ and $\frac{1}{\eta}\psi$, and bounding each of them, gives us the following:
    \begin{align*}
        \nabla^3 F_t^*(y)[h, h, h] &= - \nabla^3 F_t(x) \left[ (\nabla^2 F_t(x))^{-1} h, (\nabla^2 F_t(x))^{-1} h, (\nabla^2 F_t(x))^{-1} h \right]\\
        &= -\sum_{\tau=1}^t\nabla^3 f_\tau(x)\sbr{(\nabla^2 F_t(x))^{-1} h, (\nabla^2 F_t(x))^{-1} h, (\nabla^2 F_t(x))^{-1} h} \\&- \nabla^3\frac{1}{\eta}\psi(x)\sbr{(\nabla^2 F_t(x))^{-1} h, (\nabla^2 F_t(x))^{-1} h, (\nabla^2 F_t(x))^{-1} h}\\
        &\le \rbr{\frac{2C\eta^{3/2}}{\sigma^{3/2}}+M\eta^2}\norm{h}_2^3.
    \end{align*}
    
    Therefore, $F_t^*$ is $\rbr{\frac{2C\eta^{3/2}}{\sigma^{3/2}}+M\eta^2}$-smooth of order $3$ with respect to $\|\cdot\|_2$.
\end{proof}

\begin{lemma}\label{lem:psi-by-eta-smoothness}
   If $\psi^*$ is $M$-smooth of order $3$ with respect to $\|\cdot\|_2$, then the convex conjugate of $\frac{1}{\eta}\psi$ is $\eta^2M$-smooth of order $3$.
\end{lemma}
\begin{proof}
    By the definition of convex conjugate, we know that
    \[
    \rbr{\frac{1}{\eta}\psi}^*(y)=\sup_x \left(\langle y,x\rangle -\frac{1}{\eta}\psi(x)\right).
    \]
    Rearranging the terms, we can obtain that
    \[
    \eta\cdot\rbr{\frac{1}{\eta}\psi}^*(y)=\sup_x \left(\langle \eta y,x\rangle -\psi(x)\right)=\psi^*(\eta y).
    \]
    Taking the third derivative, we have
    \[
     \nabla^3\rbr{\frac{1}{\eta}\psi}^*(y)=\eta^2\nabla^3\psi(\eta y),
    \]
    which concludes the result in the lemma.
\end{proof}

\begin{lemma}\label{lem:max_singular_val}
    Assume that $A$ is a positive definite matrix, while $B$ is a positive semi-definite matrix. Then, it holds that the maximum singular value of $A(A+B)^{-1}$ is bounded by $1$.
\end{lemma}
\begin{proof}
    Let $C=A+B$,
    we have $C$ is a positive definite matrix and
    \[
    A(A+B)^{-1}=(C-B)C^{-1}=I-BC^{-1}.
    \]
    Since $C^{-1}$ is symmetric, there exists orthogonal matrix $M$ such that $M^{-1}C^{-1}M=D$, where $D$ is a diagonal matrix with positive diagonal entries. The singular value of $A(A+B)^{-1}$ is same as
    \[
    M^{-1}A(A+B)^{-1}M=I-B'D
    \]
    where $B'=M^{-1}BM$ is also symmetric. Since similarity transformation does not change the singular value, we only need to calculate the singular value of the following matrix:
    \[
    D^{\frac{1}{2}}(I-B'D)D^{-\frac{1}{2}}=I-D^\frac{1}{2}B'D^{\frac{1}{2}},
    \]
    where $D^\frac{1}{2}B'D^{\frac{1}{2}}$ is symmetric. Therefore, the singular value of $I-D^\frac{1}{2}B'D^{\frac{1}{2}}$ is bounded by $1$, which implies that the maximum singular value of $A(A+B)^{-1}$ is bounded by $1$.
\end{proof}

%% file: appendix_lower_bound.tex
\section{Omitted Details in \pref{sec: lower_bound}}\label{app:lower_bound}
In this section, we show omitted proofs in \pref{sec: lower_bound}. We first prove \pref{thm:hedge_lower_bound}, which shows that Hedge suffers $\Omega(T^{\frac{1}{3}})$ alternating regret in the expert problem.
\begin{proof}[Proof of \pref{thm:hedge_lower_bound}]
We prove the lower bound by constructing two environments and showing that if the Hedge algorithm achieves $\order(T^{1/3})$ alternating regret for one environment, then it will suffer $\Omega(T^{1/3})$ alternating regret for the other one.

\textbf{Environment 1}: We consider the time horizon to be $3T$ and $3$ actions with the loss vector cycling between the three basis vectors in $\mathbb{R}^3: (1,0,0), (0,1,0), (0,0,1)$ and we have $\min_{i\in[3]} \sum_{t=1}^{3T} \ell_{t,i} = T$. Direct calculation shows that Hedge with learning rate $\eta>0$ predict the $p_t$ sequence as follows: $p_{3t-2}=\rbr{\frac{1}{3},\frac{1}{3},\frac{1}{3}}, p_{3t-1}=\rbr{\frac{e^{-\eta}}{2+e^{-\eta}},\frac{1}{2+e^{-\eta}},\frac{1}{2+e^{-\eta}}}, p_{3t}=\rbr{\frac{e^{-\eta}}{1+2e^{-\eta}},\frac{e^{-\eta}}{1+2e^{-\eta}},\frac{1}{1+2e^{-\eta}}}$, $t\in[T]$. Thus, we can bound the alternating regret as follows:
\begin{align*}
   & \RegAlt = T\rbr{\frac{2}{3}+\frac{1+e^{-\eta}}{2+e^{-\eta}}+\frac{1+e^{-\eta}}{1+2e^{-\eta}}} - 2T= \frac{T(1-e^{-\eta})^2}{3(2+e^{-\eta})(1+2e^{-\eta})}\ge \frac{T(1-e^{-\eta})^2}{27}.
\end{align*}
To proceed, note that when $\eta \geq 1$, the above inequality already means that $\RegAlt=\Omega(T)$. Therefore, we only consider the case when $\eta\leq 1$. Using the fact that $e^{-\eta} \le 1-\eta+\frac{\eta^2}{2}$ for $\eta\geq 0$, we can further lower bound $\frac{T(1-e^{-\eta})^2}{27}$ by $
    \RegAlt \ge \frac{T(\eta-\frac{{\eta}^2}{2}^2)}{27}\ge \frac{\eta^2T}{108},$
where the last inequality is due to $\eta\leq 1$. Therefore, we know that $\RegAlt = \Omega(\eta^2T)$.

\textbf{Environment 2}: We consider the time horizon to be $T$ and $3$ actions with the loss vector $\ell_t$ being $(1,0,0)$ for all rounds. Here, the benchmark $\min_{i\in[3]} \sum_{t=1}^T \ell_{t,i} $ is $0$, and $p_{t,1}=\frac{e^{-\eta T}}{2+e^{-\eta T}}$ for all $t\in[T]$. In this case, if $\eta\leq \frac{2}{T}$, we know that $p_{t,1}\geq \frac{e^{-2}}{2+e^{-2}}$ for all $t\in [T]$ and the algorithm will be suffering $\Omega(T)$ regret. When $1\geq \eta\geq \frac{2}{T}$, we have:
\begin{align*}
    \RegAlt &= \sum_{t=0}^T \frac{2e^{-\eta t}}{2+e^{-\eta t}} - \frac{1}{3} -\frac{e^{-\eta T}}{2+e^{-\eta T}} \ge \sum_{t=1}^{T-1} \frac{2}{1+2e^{\eta t}}\ge \int_{1}^{T} \frac{2}{3e^{\eta t}}dt= \frac{2}{3\eta}\left[-e^{-\eta t}\right]\Big|_{1}^{T}\ge\frac{e^{-1}}{3\eta},
\end{align*}
where the second inequality uses $e^{\eta t}\geq 1$ and the last inequality uses $\eta\leq 1$.
Therefore, we have $\RegAlt = \Omega\Big(\frac{1}{\eta}\Big)$. Combining both environments, we know that Hedge with learning rate $\eta$ suffers a $\Omega(\max\{\frac{1}{\eta},\eta^2 T\})$, leadning to a $\Omega(T^{\frac{1}{3}})$ lower bound.
\end{proof}

Next, we show that OOGD and \PRM suffer $\Omega(\sqrt{T})$ alternating regret in the adversarial environment.

\begin{proof}[Proof of \pref{thm: PRM+}]
{{To prove the lower bound for both algorithms, we use the following loss sequence to show the lower bound: for $k\ge 0$, consider
    \begin{align*}
    {\ell}_{2k}=
    \begin{bmatrix}
        4\\
        0
    \end{bmatrix},\quad
    {\ell}_{2k+1}=
    \begin{bmatrix}
        -2\\
        0
    \end{bmatrix}.
\end{align*}

We start with the proof for OOGD. The process of OOGD with learning rate $\eta>0$ is as follows. Let $\wh{p}=\frac{1}{2}\cdot \bm{1}$ and $m_1=\bm{0}$ where $\bm{0}$ and $\bm{1}$ are vectors with all components equal to $0$ and $1$. At each round $t$, OOGD selects $p_t=\argmin_{p\in \Delta_2}\{\inner{p,m_t}+\frac{1}{2\eta}\|p-\wh{p}_t\|_2^2\}$. Then, after receiving $\ell_t$, OOGD updates $\wh{p}_{t+1}$ to be $\argmin_{p\in \Delta_2}\{\inner{p,\ell_t}+\frac{1}{2\eta}\|p-\wh{p}_t\|_2^2\}$ and sets $m_{t+1}=\ell_t$. 

Without loss of generality, we assume that $\frac{1}{\eta}$ and $T$ are even.
To show OOGD suffers $\Omega(\sqrt{T})$ alternating regret, direct calculation shows that when $t\leq \frac{1}{\eta}-3$, we have
\begin{align}\label{eqn:form_1}
    p_t = \begin{cases}
        [\frac{1}{2},\frac{1}{2}] &\mbox{when $t=1$},\\
        [\frac{1}{2}-k\eta+3\eta,\frac{1}{2}+k\eta-3\eta] &\mbox{when $t=2k$}, \\
        [\frac{1}{2}-k\eta-2\eta,\frac{1}{2}+k\eta+2\eta] &\mbox{when $t=2k+1$};
    \end{cases}
\end{align}
when $t\geq \frac{1}{\eta}+6$, we have
\begin{align}\label{eqn:form_2}
    p_t = \begin{cases}
        [2\eta,1-2\eta] &\mbox{when $t=2k$},\\
        [0,1] &\mbox{when $t=2k+1$}.
    \end{cases}
\end{align}
Based on the analytic form of $p_t$ as shown in \pref{eqn:form_1} and \pref{eqn:form_2}, when $t\leq \frac{1}{\eta}-3$, the alternating regret can be calculated as
\begin{align*}
    \sum_{k=1}^{\frac{1}{2\eta}-2}\inner{p_{2k}-e_2,\ell_{2k-1}+\ell_{2k}} + \inner{p_{2k+1}-e_2,\ell_{2k}+\ell_{2k+1}}=\sum_{k=1}^{\frac{1}{2\eta}-2}(2-4k\eta+2\eta) =\frac{1}{2\eta}-8\eta.
\end{align*}
When $t\geq \frac{1}{\eta}+6$, the alternating regret is
\begin{align*}
    \sum_{k=\frac{1}{2\eta}+3}^{T/2-1}\inner{p_{2k}-e_2,\ell_{2k-1}+\ell_{2k}} + \inner{p_{2k+1}-e_2,\ell_{2k}+\ell_{2k+1}}=2\eta T - 2-12\eta.
\end{align*}
Combining the above two cases, we know that
\begin{align*}
    \RegAlt = C + \frac{1}{2\eta} + 2\eta T - 2 - 20\eta,
\end{align*}
where $C$ is the regret for the remaining rounds (which is only a constant). This finishes our proof of the  $\Omega(\sqrt{T})$ alternating regret lower bound for OOGD.}}

Next, we prove our results for \PRM under the same loss sequence.
The procedure of \PRM is as follows: Let $\wh{R}_1 = {R}_1 = {r}_{0} = \bm{0}$,
and for $t\ge 1$, \PRM selects $p_t$ to be $\hat{{R}}_t/\norm{\hat{{R}}_t}_1$ where $\hat{{R}}_t = [{R}_t+{r}_{t-1}]^+$ and update ${R}_{t+1}$ to be $[{R}_t+{r}_t]^+$, where ${r}_t = \langle {p}_t, {\ell}_t\rangle \bm{1}_d - {\ell}_t$.

To simplify notation, denote $4\alpha_k=R_{2k+1,2}$ for $k\in [\frac{T}{2}-1]$. \pref{lem: PRM+} shows that for $k\ge 5$, $\alpha_k$ follows the following recurrence relation: $\alpha_{k+1}=\alpha_k+\frac{1}{1+\alpha_k}$. We use this recurrence to compute the values of $p_t$ for all rounds $t>10$.

Thus, using \pref{lem: PRM+}, the alternating loss (standard loss + cheating loss) for the rounds $t=2k+1$ and $t=2k+2$ can be calculated as
\begin{align*}
    \inner{{p}_{2k+1},\ell_{2k}+\ell_{2k+1}}+\inner{{p}_{2k+2},\ell_{2k+1}+\ell_{2k+2}}= \frac{2}{(1+\alpha_k)}.
\end{align*}
Since the action $2$ always has a loss of 0, the benchmark here is 0.
Therefore, the alternating regret is:
\begin{equation}\label{eqn:prm+-regalt}
    \RegAlt = C'+\sum_{k=5}^{\frac{T}{2}}\frac{2}{1+\alpha_k},
\end{equation}
where $C'$ is a constant bounding the regret for the first $10$ rounds.
To estimate the quantity above, we prove $\alpha_k\le 2\sqrt{k}-1$ by induction. The base case $k=5$ can be verified by direct calculation. Now, let us assume that the claim holds for $k$. Then, for $k+1$, 
\begin{align*}
\alpha_{k+1}&=\alpha_k +\frac{1}{1+\alpha_k}\le 2\sqrt{k}-1+\frac{1}{2\sqrt{k}}\le\frac{4k+1}{2\sqrt{k}}-1\le 2\sqrt{k+1}-1.
\end{align*}
The first inequality comes from the fact that the function $f(x)=x+\frac{1}{1+x}$ is monotonically increasing for $x\ge 0$.
Substituting it in \pref{eqn:prm+-regalt}, we get
\begin{align*}
\RegAlt\ge C'+\sum_{k=5}^{\frac{T}{2}} \frac{1}{\sqrt{k}}=\Theta(\sqrt{T}).
\end{align*}
Therefore, $\RegAlt = \Omega(\sqrt{T})$ for the loss sequence proposed.
\end{proof}

\begin{lemma}\label{lem: PRM+}
   Suppose that the loss vector sequence satisfies that $\ell_{2k}=\begin{bmatrix}
    4\\ 0
\end{bmatrix},{\ell}_{2k+1}=\begin{bmatrix}
    -2\\ 0
\end{bmatrix}$ for $k\geq 0$. Then, \PRM guarantees that for $k\geq 5$
\begin{align*}
&{p}_{2k+1}=\begin{bmatrix}
    0\\ 1
\end{bmatrix},{p}_{2k+2}=\begin{bmatrix}
    \frac{1}{1+\alpha_k}\\ \frac{\alpha_k}{1+\alpha_k}
\end{bmatrix},\\
&{R}_{2k+2}=\begin{bmatrix}
    2\\ 4\alpha_k
\end{bmatrix},R_{2k+3}=\begin{bmatrix}
    0\\ 4\alpha_k+\frac{4}{1+\alpha_k}
\end{bmatrix},\\
&\wh{R}_{2k+2}=\begin{bmatrix}
    4\\ 4\alpha_k
\end{bmatrix},\wh{R}_{2k+3}=\begin{bmatrix}
    0\\ 4\alpha_k+\frac{8}{1+\alpha_k}
\end{bmatrix},
\end{align*}
where $\alpha_5>2$ is certain constant and $\alpha_{k+1}=\alpha_k+\frac{1}{1+\alpha_k}$ for $k\geq 5$.
\end{lemma}

\begin{proof}
It can be verified that \pref{lem: PRM+} holds true when $k=5$. For $k>5$, we prove by induction. Suppose \pref{lem: PRM+} holds for $k$. Then, for $k+1$, we have,
\begin{align*}
    &{p}_{2k+3} =\frac{\wh{R}_{2k+3}}{\norm{\wh{R}_{2k+3}}_1} =\begin{bmatrix}
        0\\ 1
    \end{bmatrix},  \\
    &R_{2k+4} = [R_{2k+3}+r_{2k+3}]^+ = \begin{bmatrix}
        2\\
        4\alpha_k+\frac{4}{1+\alpha_k}
    \end{bmatrix} =
    \begin{bmatrix}
        2\\
        4\alpha_{k+1}
    \end{bmatrix},
    \\
    &r_{2k+3} = \inner{p_{2k+3},\ell_{2k+3}}\mathbf{1}_d - \ell_{2k+3} = \begin{bmatrix}
        2\\
        0
    \end{bmatrix},\\
    &\hat{R}_{2k+4} = [R_{2k+4}+r_{2k+3}]^+ = 
    \begin{bmatrix}
        4\\
       4\alpha_{k+1} 
    \end{bmatrix}.
\end{align*}
Since $\alpha_k\ge 2$, we have $\alpha_{k+1}\ge 2$ as well. Using this, we can see that
\begin{align*}
    &{p}_{2k+4} =
    \frac{\hat{R}_{2k+4}}{\|\hat{R}_{2k+4}\|_1} = 
    \begin{bmatrix}
         \frac{1}{1+\alpha_{k+1}}\\
         \frac{\alpha_{k+1}}{1+\alpha_{k+1}}
    \end{bmatrix}, \\
    &r_{2k+4} = \inner{p_{2k+4},\ell_{2k+4}}\mathbf{1}_d - \ell_{2k+4} = \begin{bmatrix}
        -\frac{4\alpha_{k+1}}{1+\alpha_{k+1}}\\
        \frac{4}{1+\alpha_{k+1}}
    \end{bmatrix}, \\
    &R_{2k+5} = [R_{2k+4}+r_{2k+4}]^+ = 
    \begin{bmatrix}
        0\\
        4\alpha_{k+1}+\frac{4}{1+\alpha_{k+1}}
    \end{bmatrix}, \\
    &\hat{R}_{2k+5} = [R_{2k+5}+r_{2k+4}]^+ = 
    \begin{bmatrix}
        0\\
        4\alpha_{k+1}+\frac{8}{1+\alpha_{k+1}}
    \end{bmatrix},
\end{align*}
where the third equality uses the fact that $\frac{\alpha_{k+1}}{1+\alpha_{k+1}}\geq \frac{2}{3}>\frac{1}{2}$. Thus, the claim is true for $k+1$, and hence it holds for all $k\ge 5$.

\end{proof}